\documentclass[11pt,journal,onecolumn]{IEEEtran}

\usepackage{amsmath, amsfonts, amssymb,  epsfig, mathtools, caption}
\usepackage{times, amsbsy,bm,amsthm, bbm}
\usepackage{framed} 
\usepackage{adjustbox}
\usepackage{algorithm}
\usepackage{algorithmic}
\usepackage{multirow, booktabs, hhline}
\usepackage{subcaption, xcolor}
\usepackage{threeparttable}
\usepackage{pifont}
\def \bdelta{\bm{\delta}}

\def \bdelta{\bm{\delta}}

\DeclareMathOperator*{\argmax}{argmax}

\newtheorem{theorem}{Theorem}
\newtheorem{lemma}{Lemma}

\newtheorem{assumption}{Assumption}

\begin{document}

%\title{CMF: Contextual-Modulated Feedback Approach for Online Continual Learning}
\title{G-Mix: A Generalized Mixup Learning Framework Towards Flat Minima}

\author{Xingyu~Li,~and~Bo~Tang,~\IEEEmembership{Senior Member,~IEEE}\thanks{Copyright (c) 2015 IEEE. Personal use of this material is permitted. However, permission to use this material for any other purposes must be obtained from the IEEE by sending a request to pubs-permissions@ieee.org.}
	\thanks{Xingyu Li is with the Department of Electrical and Computer Engineering, Mississippi State University, Mississippi State, MS, 39762 USA. e-mail: xl292@msstate.edu. Bo Tang is with the Department of Electrical and Computer Engineering at Worcester Polytechnic Instiute, Worcester, MA. e-mail: btang1@wpi.edu. This research is partially supported by the National Science Foundation, under grant number IIS-2047570, and United States Army Research Office, under grant number W911NF2110290.}%
	% 	\thanks{Zhe~Qu and Zhuo~Lu are with the Department of Electrical Engineering at University of South Florida, Tampa, FL, 33620. E-mail: zhequ@usf.edu, zhuolu@usf.edu.}
	% 	\thanks{Xingyu~Li and Zhe~Qu are co-first authors.}
	% 	\thanks{Haifeng Li is with the School of Geosciences and Info-Physics, Central South University, Changsha 410083, China e-mail: lihaifeng@csu.edu.cn.}%
}

\markboth{Journal of IEEE Transactions on Artificial Intelligence, Vol. 00, No. 0, Month 2020}
{First A. Author \MakeLowercase{\textit{et al.}}: Bare Demo of IEEEtai.cls for IEEE Journals of IEEE Transactions on Artificial Intelligence}

\maketitle
%\IEEEtitleabstractindextext{%
	\begin{abstract}
Deep neural networks (DNNs) have demonstrated promising results in various complex tasks. However, current DNNs encounter challenges with over-parameterization, especially when there is limited training data available. To enhance the generalization capability of DNNs, the Mixup technique has gained popularity. Nevertheless, it still produces suboptimal outcomes. Inspired by the successful Sharpness-Aware Minimization (SAM) approach, which establishes a connection between the sharpness of the training loss landscape and model generalization, we propose a new learning framework called Generalized-Mixup, which combines the strengths of Mixup and SAM for training DNN models. The theoretical analysis provided demonstrates how the developed G-Mix framework enhances generalization. Additionally, to further optimize DNN performance with the G-Mix framework, we introduce two novel algorithms: Binary G-Mix and Decomposed G-Mix. These algorithms partition the training data into two subsets based on the sharpness-sensitivity of each example to address the issue of "manifold intrusion" in Mixup. Both theoretical explanations and experimental results reveal that the proposed BG-Mix and DG-Mix algorithms further enhance model generalization across multiple datasets and models, achieving state-of-the-art performance.

% \begin{IEEEImpStatement}
	The generalization problem of the over-parameterization deep neural networks is popular since the rapidly developed large models recently. To address this, this study introduces Generalized-Mixup (G-Mix). Inspired from the successful Sharpness-Aware Minimization (SAM) method and the Mixup method, G-Mix combines their strengths, demonstrating improved performance, especially when training data is limited. Further innovations in this research include the development of two novel algorithms, Binary G-Mix and Decomposed G-Mix, that intelligently partition training data based on each example's sharpness-sensitivity. This approach effectively resolves the "manifold intrusion" problem commonly experienced in Mixup. Both theoretical and experimental evidence underscore the potential of these algorithms to boost the generalization ability of DNNs across various datasets and models, providing a new benchmark in the field.
% \end{IEEEImpStatement}
	\end{abstract}
	

%\maketitle

%\IEEEdisplaynontitleabstractindextext

%\IEEEpeerreviewmaketitle

%\ifCLASSOPTIONcompsoc
%\IEEEraisesectionheading{\section{Introduction}\label{sec:introduction}}
%\else
\section{Introduction}\label{sec:introduction}

\IEEEPARstart{C}{urrent} deep neural networks (DNNs) have achieved remarkable success in various domains, including computer vision \cite{he2016deep,yu2020sprnet,li2021fedlga,wang2022rotation}, speech recognition \cite{stewart2013robust,nassif2019speech}, reinforcement learning \cite{silver2016mastering,bai2019adaptive,li2020adaptive,nguyen2020deep}, and natural language processing \cite{devlin2019bert,wang2020emotion,zheng2020predicting,sakai2021riemannian}. It is worth noting that these successful approaches often employ over-parameterized DNNs. For instance, compared to the sizes of standard image classification benchmarks like CIFAR-100 \cite{krizhevsky2009learning} (with 60k images) and ImageNet-$1$k \cite{krizhevsky2012imagenet} (with 1.2 million images), the number of parameters in DNNs is significantly larger. For example, AlexNet \cite{krizhevsky2012imagenet} has 62 million parameters, while the recently developed ResNet-50 \cite{he2016deep} and BERT \cite{devlin2019bert} have 23 million and 340 million parameters, respectively. 

Although this over-parameterized strategy enables DNNs to accurately memorize the training data distribution and converge based on the Empirical Risk Minimization (ERM) principle \cite{vapnik1999nature}, it can lead to limited generalization ability and over-confident predictions when the testing data distribution differs from the training data. This phenomenon is known as overfitting \cite{srivastava2014dropout,zhang2021understanding}. Recent studies \cite{hoffer2017train,neyshabur2018pac,zhang2021how} have indicated that the quantified difference between testing and training accuracy during the training process serves as an important indicator of the generalization gap.

To address the challenges of generalization and robustness in DNNs, recent research has proposed various approaches that fall into two categories: data augmentation techniques and modified training algorithms. The first category focuses on augmenting the training data \cite{srivastava2014dropout,devries2017improved,zhong2020random,zeng2020regularization} to improve generalization. One effective method in this category is Mixup \cite{zhang2018mixup}, which creates new examples by linearly interpolating training pairs of examples and labels. Mixup has demonstrated significant empirical improvements. The second category involves modifying the training algorithms to prevent the model from converging into sub-optimal solutions. This can be achieved through techniques such as penalization \cite{loshchilov2018decoupled}, normalization \cite{wu2018group}, or perturbation \cite{lin2020extrapolation,zheng2021regularizing}.

While Mixup has shown significant improvements in generalization, it suffers from suboptimal solutions due to its incompatibility with certain training algorithms. In a recent study \cite{foret2020sharpness}, Sharpness-Aware Minimization (SAM) was proposed to reduce the empirical generalization error by leveraging the connection between the generalization ability and a ``flat" training loss landscape. Surprisingly, SAM and Mixup have orthogonal relationships, making their modified training schemes compatible. Hence, there is potential to further enhance the generalization ability of DNN models by integrating these two techniques into a single training framework.

Motivated by this, we introduce a new learning framework called Generalized-Mixup (G-Mix) in this paper. G-Mix integrates the advantages of both Mixup and SAM for DNN model training. Under the G-Mix framework, we enhance the generalization ability of DNNs through two procedures: (i) a data augmentation procedure that employs linear convex interpolation on the given training data, transforming the training loss principle from Empirical Risk Minimization (ERM) to Vicinal Risk Minimization (VRM), and (ii) a weight perturbation procedure that utilizes the model learned by adding worst-case perturbations of VRM loss as a regulator during the training process. Notably, to execute the worst-case perturbation, the G-Mix framework performs an additional back-propagation operation for the perturbed content approximation. Empirical validations demonstrate that the G-Mix framework significantly improves the performance of DNNs compared to both Mixup and SAM. However, despite the impressive empirical performance of Mixup and SAM, the underlying reasons for these improvements have not been fully understood and there remain open problems \cite{carratino2020mixup,gao2020consensus,zhang2021how,shang2021alpha,zhuang2022surrogate,du2022efficient}. To enable theoretical analysis, we relate the loss function induced by G-Mix to the standard ERM training loss with two additional regularization terms from Mixup and SAM. Our theoretical analysis shows how these regularization terms are Lipschitz continuous and can improve the generalization ability of DNNs compared to standard training methods.

Furthermore, we enhance the generalization capability of DNNs through the two-step G-Mix framework to address the ``manifold intrusion" problem. As discussed in the literature \cite{verma2019manifold,guo2019mixup,mai2021metamixup}, ``manifold intrusion" refers to situations where the generated examples from Mixup can conflict with the original training data. We investigate the relationship between ``manifold intrusion" and the weight perturbation process in the proposed G-Mix framework. Specifically, if a training example is more sensitive to loss changes due to weight perturbation, its training loss direction is closer to the ``flat minima" in SAM, reducing the likelihood of conflicts related to ``manifold intrusion."

Based on this insight, we divide the training examples after the Mixup procedure in G-Mix into two subsets: sensitive and less-sensitive. This division is accomplished by calculating a sorted sharpness-sensitivity score. Using this division, we propose two novel algorithms: Binary G-Mix (BG-Mix) and Decomposed G-Mix (DG-Mix). The BG-Mix algorithm disregards the less-sensitive examples in the training batch during the learning process to address ``manifold intrusion" and improve G-Mix computational efficiency. The DG-Mix algorithm decomposes the loss of less-sensitive examples into two components: one parallel and one orthogonal to the averaged direction of the sensitive examples. It then performs an additional back-forward operation on the orthogonal component as a regularization term. Our theoretical and empirical results show that the proposed BG-Mix and DG-Mix algorithms can further improve the performance of DNNs.

To the best of our knowledge, our research is the first to utilize the sensitivity of loss sharpness to guide the interpolation policy in the Mixup method. This algorithm effectively addresses the issue of manifold intrusion without incurring additional computational costs. In summary, the main contributions of this paper are as follows:
\begin{itemize}
	\item \textbf{G-Mix framework.} We introduce the two-step G-Mix learning framework, which combines the strengths of existing Mixup and SAM approaches. The theoretical analysis provided demonstrates that the developed G-Mix framework significantly improves the generalization ability of DNNs. 
	\item \textbf{BG-Mix and DG-Mix algorithms.}  To further enhance the generalization ability and tackle the ``manifold intrusion" problem within the G-Mix framework, we propose two novel algorithms: BG-Mix and DG-Mix. These algorithms provide a binary/decomposed implementation on the training losses of sensitive and less-sensitive examples, respectively, based on the corresponding sharpness-sensitivity calculations.
	\item \textbf{Empirical results.} We conduct extensive experiments on multiple datasets and models, showcasing that the proposed BG-Mix and DG-Mix algorithms outperform existing representative approaches and achieve state-of-the-art performance.
\end{itemize}
The rest of this paper is organized as follows. Sec.~\ref{Sec:relatedwork} describes the literature relevant to our work. Sec.~\ref{Sec:method} details our G-Mix framework, followed by a thorough theoretical analysis. We then illustrate the proposed BG-Mix and DG-Mix algorithms in Sec.~\ref{Sec:algorithm}, and the comprehensive experimental analysis is introduced in Sec.~\ref{Sec:analysis}, followed by a conclusion in Sec.~\ref{Sec:conclusion}. 

\begin{figure*}[t!]
	\centering
	\includegraphics[width = 0.85\columnwidth]{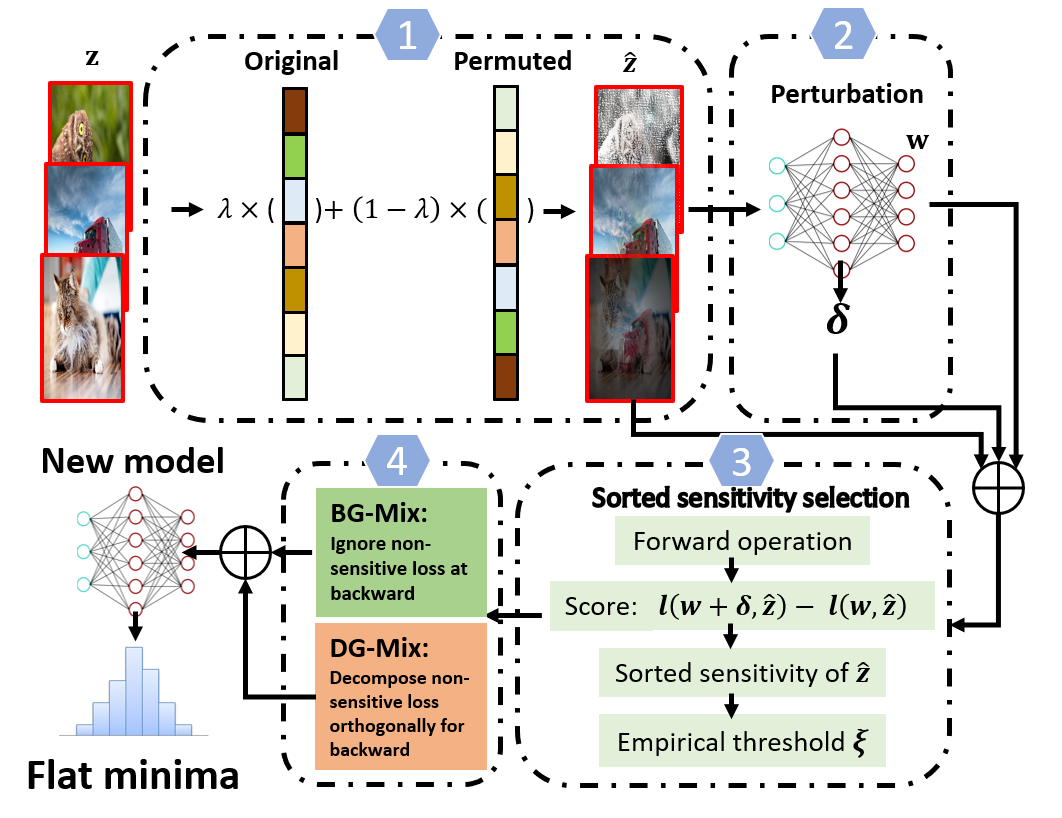}
	\captionsetup{justification=centering}
	
	\caption{The illustration of the proposed G-Mix learning framework and two extended algorithms: BG-Mix and DG-Mix algorithms.}
	\label{fig:overview}
\end{figure*}

\section{Related Works}\label{Sec:relatedwork}
The training process of neural network models requires optimization to enhance their generalization ability and mitigate the issue of overfitting \cite{hochreiter1994simplifying, bishop2006pattern}. With the remarkable success of deep neural networks (DNNs) in various domains \cite{krizhevsky2012imagenet, he2016deep, nassif2019speech, devlin2019bert}, the generalization error arising from the mismatch between testing and training data distributions in over-parameterized DNN models has become increasingly important. Within the framework of gradient descent-based deep learning approaches, efforts to improve generalization ability can be broadly categorized into two main approaches: data augmentation and modified training algorithms.

\subsection{Data Augmentation}

Data augmentation techniques have gained significant importance in recent successful deep learning (DL) applications. These techniques leverage domain knowledge of the dataset to design appropriate transformations for training examples. Existing data augmentation methods employ a range of transformations, including rotation, translation, cropping, resizing, mapping, noise injection \cite{bishop1995training}, horizontal flipping \cite{krizhevsky2012imagenet}, and random erasing \cite{zhong2020random}. These transformations aim to increase the valuable information within the training examples, thereby improving the generalization ability of deep neural networks (DNNs).

One prominent data augmentation technique is Mixup, proposed by \cite{zhang2018mixup}. Mixup generates synthetic training examples by linearly interpolating between pairs of input data samples and their corresponding labels. This training strategy has empirically demonstrated enhanced generalization across a wide range of tasks, including supervised \cite{lamb2019interpolated, guo2019mixup, verma2019manifold} and semi-supervised \cite{verma2019manifold, mai2021metamixup} learning scenarios. The success of Mixup has motivated the development of several variant techniques.

For instance, Manifold Mixup extends the convex interpolation from the input data space to the feature space \cite{verma2019manifold}. AdaMixup addresses the ``manifold intrusion" problem in Mixup by learning the interpolation policy using an additional network \cite{guo2019mixup}. CutMix generates synthetic examples by blending a random rectangular region from one sample into another \cite{yun2019cutmix}. Additionally, PuzzleMix \cite{kim2020puzzle} and MetaMixup \cite{mai2021metamixup} dynamically learn the interpolation policy of Mixup by leveraging saliency information \cite{Simonyan14a} and meta-learning optimization \cite{finn2017model}, respectively.

These various data augmentation techniques and their extensions aim to exploit the inherent structures and patterns in the data to improve the generalization capability of DNNs. By incorporating domain-specific knowledge into the training process, these methods help DNNs learn more robust and representative features, ultimately leading to improved performance on a variety of tasks.
\subsection{Modified Training Algorithms}
In addition to data-dependent augmentation techniques, another approach to improving the generalization of DNNs is through modifications to existing training methods. These modification approaches introduce new objectives or regularization techniques to the standard training process. They aim to provide additional constraints or guidance to the model learning process, leading to improved generalization.

Several regularization methods have been proposed in the literature. For example, weight decay, introduced by \cite{loshchilov2018decoupled}, regulates the training loss by adding a penalty term to the weight updates. Dropout, proposed by \cite{srivastava2014dropout}, modifies the model architecture by randomly dropping out units during training. Shake-Shake regularization, presented by \cite{gastaldi2017shake}, involves stochasticity in the model's skip connections. Normalization techniques such as batch normalization \cite{ioffe2015batch} and group normalization \cite{wu2018group} have also been utilized during training. Additionally, weight perturbation methods, like those proposed by \cite{lin2020extrapolation, zheng2021regularizing}, apply perturbations to the weights of the DNN model to improve generalization.

Recent studies \cite{lecun2012efficient, keskar2017large, zhang2021understanding} have observed a positive correlation between the generalization ability of over-parameterized DNNs and the sharpness of the loss landscape. Motivated by this phenomenon, \cite{foret2020sharpness} introduced the Sharpness-Aware Minimization (SAM) method. SAM adds a maximized weight perturbation to leverage the connection between minimizing the generalization error and finding "flat" minima in the loss landscape. Similarly, \cite{zheng2021regularizing} proposed a method similar to SAM, which improves generalization on various DL benchmarks through weight perturbation. Moreover, \cite{du2022efficient} focused on enhancing the efficiency of SAM implementation, while \cite{zhuang2022surrogate} introduced an auxiliary objective function to provide a guarantee for finding flat minima.

These modification approaches provide alternative training strategies that aim to improve the generalization ability of DNNs. By incorporating additional regularization techniques and leveraging the properties of the loss landscape, they offer new perspectives and avenues for achieving better generalization performance in deep learning.

\section{G-Mix Framework}\label{Sec:method}
In this section, we present the proposed G-Mix learning framework, along with the BG-Mix and DG-Mix algorithms explained in Section \ref{Sec:algorithm}. To provide a clear overview of our approach, we depict the framework and algorithms in Figure \ref{fig:overview}. The steps involved in the two-step G-Mix learning framework and the proposed algorithms are as follows: 

(1). Augmentation procedure: The augmentation procedure involves data augmentation on the given training data examples through random linear interpolation. (2). Weight perturbation procedure: The weight perturbation procedure adds approximated weight perturbations to the model. This is achieved by performing an additional backward propagation operation using the augmented data obtained in step (1). (3). Sensitivity selection: The sensitivity selection step involves determining the sensitivity of the linearly interpolated augmented data. This is accomplished by leveraging the previously obtained weight perturbations. The sensitivity measures the response of the loss to changes in the weights and is used to select the appropriate augmented data. (4). Achieving ``flat minima": The final step aims to achieve ``flat minima" by utilizing the updated model after applying the proposed binary or decomposed algorithms. These algorithms, namely BG-Mix and DG-Mix, ensure that the model converges to regions in the loss landscape with flatter minima, thereby improving the generalization ability of the DNN.

\subsection{Notations}
%We first state the major notations in our work and briefly recap the definition of the Mixup method for better understanding. 
Suppose that $\mathcal{X}$ denotes the input dataset where each training example $\bm{x}$ has $p$ features, and $y \in \mathcal{Y}$ is the set of labels with $m$ classes. The general parameterized loss used in this paper is defined as $l(\bm{w}, \bm{z})$, where $\bm{w}$ denotes $d$-dimension parameters of a trainable DNN model that $\bm{w} \subseteq \mathbb{R}^{d}$, and $\bm{z} = (\bm{x}, y)$ represents the pair of input data and its associated output label. Hence, the given dataset of $n$ training examples is $\mathcal{Z} = \{\bm{z}_1, \dots, \bm{z}_n\}$, which is indexed by $i$ that $\bm{z}_i = (\bm{x}_i, y_i)$ could be regarded as i.i.d. drawn from a joint distribution $\mathcal{P}_{\bm{x},y} = \mathcal{D}$, satisfying $\bm{x}_i \in \mathcal{X} \subseteq \mathbb{R}^{p}$ and $y_i \in \mathcal{Y} \subseteq \mathbb{R}^{m}$.

Moreover, the standard empirical risk minimization (ERM) loss could be defined as $L_{\mathcal{S}} (\bm{w}, \bm{z}) \triangleq \mathbb{E}_{\bm{z} \sim \mathcal{S}} l(\bm{w}, \bm{z}) = \frac{1}{n} \sum_{i=1}^{n} l(\bm{w}_i, \bm{z}_i)$, and the corresponding population loss is $L_{\mathcal{D}} (\bm{w}, \bm{z}) \triangleq \mathbb{E}_{\bm{z} \sim \mathcal{D}} [l (\bm{w}, \bm{z})]$. Note that in this case, as $\mathcal{D}$ is not known to the network, using $L_{\mathcal{S}}$ as an estimation of $L_{\mathcal{D}}$ becomes the major motivation to train modern DNN models, such as SGD and Adam. Additionally, we denote the gradient of a general loss objective $L$ concerning $\bm{z}$ and $\bm{w}$ as  $\nabla_{\bm{z}} L(\bm{w}, \bm{z}) $ and $\nabla_{\bm{w}} L(\bm{w}, \bm{z})$ for model updates. 

\subsection{Two-Step G-Mix Framework with Mixup and SAM}

In this section, we introduce our proposed new learning framework, called Generalized-Mixup (G-Mix), which combines the advantages of Mixup and SAM. As depicted in Figure \ref{fig:overview}, the G-Mix framework aims to approximate the desired population loss $L_{\mathcal{D}}$ using the given training data inputs. The framework consists of two main steps, as shown in Step 1 and Step 2 in Figure \ref{fig:overview}.

Step 1: Integration of Mixup and SAM: In this step, we formulate a new training loss objective by integrating Mixup into SAM. This integration leverages the strengths of both techniques to improve the generalization performance of DNNs. Mixup generates synthetic training examples through linear interpolation, while SAM enhances the generalization ability by exploring the sharpness of the loss landscape. By combining these two methods, we create a new learning framework that improves upon the limitations of individual approaches.

Step 2: Augmentation and weight perturbation: This step involves the augmentation of training data examples through linear interpolation, similar to Mixup. The augmented data is then used to perform weight perturbation, similar to SAM, by adding approximated weight perturbations through an additional backward propagation operation. This procedure enhances the model's ability to generalize by exploring the "flat minima" in the loss landscape.

The G-Mix framework, as described in these two steps, provides a comprehensive learning framework that enhances the generalization performance of DNNs. The theoretical analysis and empirical studies conducted on the framework demonstrate its effectiveness. Furthermore, based on the G-Mix framework, we develop two novel learning algorithms, as shown in steps 3 and 4 in Figure \ref{fig:overview}. These algorithms, called Binary G-Mix (BG-Mix) and Decomposed G-Mix (DG-Mix), further refine the training process to address the ``manifold intrusion" problem and improve the generalization ability of the DNN models. The details of these algorithms are presented in the next section.

% Next, we introduce the proposed new learning framework, called $\underline{G}$eneralized-$\underline{Mix}$up. As illustrated in Fig.~\ref{fig:overview}, to approach the desired population loss $L_{\mathcal{D}}$ from the given training data inputs, the developed G-Mix framework first formulates a new training loss objective by integrating the Mixup into the SAM, as shown in Step 1 and Step 2 in Fig. \ref{fig:overview}. This subsection starts by detailing this two-step learning framework. Both our theoretical analysis and empirical studies show improved generalization performance. Also based on this framework, two novel learning algorithms are further developed, as shown in steps 3 and 4 in Fig. \ref{fig:overview}, which are detailed in the next section.  

\subsubsection{Step of Mixup}
Given two labeled  training data pairs as $\bm{z}_i = (\bm{x}_i, y_i)$ and $\bm{z}_j = (\bm{x}_j, y_j)$, the Mixup \cite{zhang2018mixup} uses a linear convex interpolation method to generate new data: $\hat{\bm{x}}_{i,j} (\lambda) = \lambda \bm{x}_i + (1-\lambda) \bm{x}_j$ and $\hat{y}_{i,j} (\lambda) = \lambda y_i + (1-\lambda) y_j$, where $\lambda \in [0,1]$ is randomly drawn from a Beta distribution Beta$(\alpha, \alpha)$, $\alpha \in (0, \infty)$. With newly generated data pairs $\bm{\hat{z}}_{i,j} (\lambda) = (\bm{\hat{x}}_{i,j} (\lambda), {\hat{y}}_{i,j} (\lambda))$, $i,j=1, \dots, n$, the ERM loss objective $L_{\mathcal{S}}$ would turn into the VRM loss $L^{mix}$ as
\begin{equation}\label{Eq:mixup}
	L^{Mix} (\bm{w}, \bm{\hat{z}}) =  \frac{1}{n^{2}} \sum_{i, j=1}^{n} \mathbb{E}_{\lambda \sim \mathcal{D}_\lambda} l( \bm{w}, \bm{\hat{z}}_{i,j}), 
\end{equation}
where $\mathcal{D}_\lambda$ is the distribution of linearly interpolated $\mathcal{D}$ due to $\lambda$. Note that as the parameter settings of $\lambda$ are not our focus in this work, we follow the original work \cite{zhang2018mixup} that set the Beta$(\alpha, \alpha)$ with $\alpha = 1$ by default.

\subsubsection{Step of SAM}   
Next, we further improve the obtained loss objective $L^{mix}$ using the SAM technique. Unlike the Mixup method which focuses on the interpolation of the given training examples, the SAM method improves the generalization ability of DNNs by minimizing the sharpness of the training loss landscape. To achieve this, instead of updating model weights $\bm{w}$ by simply minimizing $L^{mix}$, the developed G-Mix framework finds parameter values of $\bm{w}$ to ensure that the neighborhoods around $\bm{w}$ within a constrained radius can have uniformly minimized loss. Note that the loss objective discussed in this paper could be many forms of functions, such as linear and logistic regression, and also non-convex cross-entropy for DNNs, which can have multiple global minima solutions that share significantly different generalization abilities towards $\mathcal{L}_{\mathcal{D}}$. 

With the SAM technique, the G-Mix framework formulates the $L^{mix}$ loss function into the following min-max optimization problem 
\begin{equation}\label{Eq:loss_gmix}
	L^{GMix}(\bm{w}, \bm{\hat{z}}) = \min_{\bm{w}} \max_{\|\bdelta \|_2 \leq \rho} L^{Mix} (\bm{w}+\bdelta, \bm{\hat{z}} ),
\end{equation}  
where with a given model $\bm{w}$, the inner optimization operation attempts to find $\bdelta$, which is the weight perturbation that maximizes the VRM loss of $L^{Mix}$ within a Euclidean ball of radius $\rho$ around $\bm{w}$. For better presentation, we define the maximized loss around $\bm{w}$ to be the sum of the Vicinal risk loss and the sharpness as $R^{GMix} (\bm{w}, \bm{\hat{z}}) = \max_{\bdelta: \|\bdelta \|_2 \leq \rho} L^{Mix} (\bm{w}+\bdelta, \bm{\hat{z}} ) - L^{Mix} (\bm{w}, \bm{\hat{z}})$. As such, $R^{GMix} (\bm{w}, \bm{\hat{z}})$ represents the quantified maximal change of the corresponding VRM loss under $L^{Mix}$ with the added constrained perturbation $\bdelta$ at $\bm{w}$, which encourages the G-Mix framework to find the desired flat minima for generalization improvement. 

As justified in \cite{foret2020sharpness, kwon2021asam}, given a certain state of $\bm{w}$, the desired population loss $L_{\mathcal{D}}$ of the DNN model could be upper-bounded by the combination of the previously obtained $L^{Mix}$, sharpness $R^{GMix}$, and a regularization term as
\begin{equation}\label{Eq:loss_upperbound}
	L_{\mathcal{D}} (\bm{w} + \bdelta, \bm{\hat{z}}) \leq R^{GMix} (\bm{w}, \bm{\hat{z}}) + L^{Mix} (\bm{w},\bm{\hat{z}}) +  \mathcal{R}(\|\bm{w} \|_2^{2}), 
\end{equation}
where $\mathcal{R}(\|\bm{w} \|_2)$ represents the regularization term which is on the norm of $\bm{w}$. Note that in this paper, we omit the discussion of the regularization term $\mathcal{R}(\|\bm{w} \|_2)$ and focus on the first two components on the right side of Eq.~\eqref{Eq:loss_upperbound}, where $R^{GMix}$ and $L^{Mix}$ are considered to be the contributions from the SAM and Mixup techniques, respectively. 

The main challenge in solving the min-max optimization problem in Eq.~\eqref{Eq:loss_gmix} is to efficiently find the maximized value $\bdelta$ within the constraint radius $\rho$. In this work, we first perform one back-propagation operation on $L^{Mix}(\bm{w}, \bm{\hat{z}})$ to obtain the gradient $\bm{w}$. Then, we approximate the inner maximization content $\bdelta$ by applying the first-order Taylor expansion at $L^{Mix}(\bm{w},\bm{\hat{z}})$ as 
\begin{equation}\label{Eq:delta} 
	\begin{split}
		\bdelta & \triangleq \argmax_{\|\bdelta\|_2 \leq \rho} L^{Mix}(\bm{w}+ \bdelta, \bm{\hat{z}} ) \\
		& = \argmax_{\|\bdelta \|_2 \leq \rho} L^{Mix} (\bm{w},\bm{\hat{z}}) + \bdelta^{\top} \nabla_{\bm{w}} L^{Mix} (\bm{w},\bm{\hat{z}}) + \mathcal{O}(\rho^2) \\
		&  \approx \argmax_{\|\bdelta \|_2 \leq \rho} L^{Mix} (\bm{w},\bm{\hat{z}}) + \bdelta^{\top} \nabla_{\bm{w}} L^{Mix} (\bm{w},\bm{\hat{z}}) \\
		& \overset{(a)}{=} \rho \frac{\nabla_{\bm{w}} L^{Mix}  (\bm{w},\bm{\hat{z}})}{\|\nabla_{\bm{w}} L^{Mix}(\bm{w},\bm{\hat{z}})\|_2},\\
	\end{split}
\end{equation}
where $\mathcal{O}(\rho^2)$ represents higher terms of the Taylor expansion at $L^{Mix} (\bm{w})$ which is omitted in our work. The last step in (a) holds because $\bdelta$ has the same direction of $ L^{Mix}  (\bm{w})$ where $\frac{\nabla_{\bm{w}} L^{Mix} (\bm{w},\bm{\hat{z}})}{\|\nabla_{\bm{w}} L^{Mix}(\bm{w},\bm{\hat{z}})\|_2} \leq 1$ and the result of $\bdelta$ satisfies the maximal constrain $\|\bdelta\| \leq \rho$, as illustrated in \cite{zhuang2022surrogate}. As such, the min-max optimization problem in Eq.~\eqref{Eq:loss_gmix} can be reduced into a perturbed weight optimization objective as 
\begin{equation}\label{Eq:gmix}
	\min_{\bm{w}} L^{GMix} (\bm{w},\bm{\hat{z}}) \approx \min_{\bm{w}}  \frac{1}{n^{2}} \sum_{i, j=1}^{n} \mathbb{E}_{\lambda \sim \mathcal{D}_\lambda} l( \bm{w} + \bdelta, \bm{\hat{z}_{i,j}}). 
\end{equation}
To this end, we can summarize the developed two-step G-Mix learning framework in Algorithm.~\ref{alg:framework}, when one training batch $\mathcal{B} \in \mathcal{S}$ is given. Specifically, during the Mixup procedure we perform the linear interpolation by uniformly permuting the index of $\mathcal{B}$, with a randomly generated hyper-parameter $\lambda$ from Beta$(\alpha, \alpha)$. And for the SAM procedure, we first use one extra back-propagation operation on $ L^{Mix}  (\bm{w},\bm{\hat{z}})$ obtained from the Mixup procedure to compute $\bdelta$, then update the model $\bm{w}$ at lines $(17)-(18)$ for solving the formulated minimization problem of G-Mix in Eq.~\eqref{Eq:gmix}, where $\bm{g}$ is the shorthand for $\bm{g} = \nabla_{\bm{w}} L^{Mix} (\bm{w} + \rho \frac{\nabla_{\bm{w}} L^{Mix}  (\bm{w},\bm{\hat{z}})}{\|\nabla_{\bm{w}} L^{Mix}(\bm{w},\bm{\hat{z}})\|_2})$.

\begin{algorithm}[tb]
	\caption{Two-step G-Mix framework}
	\label{alg:framework}
	\begin{algorithmic}[1]
		\STATE {\bfseries Input:} 
		Given training batch $\mathcal{S} = \{\bm{z}_1, \dots, \bm{z}_i, \dots, \bm{z}_n\}$, $\bm{z}_i = (\bm{x}_i, y_i)$; $\alpha$ (the beta distribution parameter); learning rate $\eta$; empirical radius $\rho$; maximum training epoch $T$.
		
		\FOR {$\text{each training epoch t from 0 to T}$}
		\IF {$\text{training epoch t is 0}$}
		\STATE Initialize model $\bm{w}^{0}$.
		\ENDIF
		
		\STATE $\underline{\text{Mixup procedure:}}$
		\STATE Sample a training batch $\mathcal{B} = \{\bm{z}_i\} \in \mathcal{S}$.
		\STATE Generate $\lambda = \text{numpy.random.beta}(\alpha, \alpha)$.
		\STATE Permute $\text{index} = \text{torch.randperm}(\bm{z}.size(0))$.
		\STATE $\bm{z}_j = \bm{z}_i[index]$: $\bm{x}_j = \bm{x}_i[index]$ and $y_j = y_i[index]$.
		\STATE $\hat{\bm{x}}_{i,j} (\lambda) = \lambda \bm{x}_i + (1-\lambda) \bm{x}_j$.
		\STATE $\hat{y}_{i,j} (\lambda) = \lambda y_i + (1-\lambda) y_j$.
		\STATE Update the loss objective from $L_{\mathcal{S}}$ to $L^{Mix}$ using Eq.~\eqref{Eq:mixup}.			
		\STATE $\underline{\text{SAM procedure:}}$
		\STATE Perform one back-propagation on $L^{Mix} (\bm{w}, \bm{\hat{z}})$.
		\STATE Compute $\bdelta$ with the $\nabla_{\bm{w}} L^{Mix} (\bm{w},\bm{\hat{z}})$ using Eq.~\eqref{Eq:delta}.
		\STATE Compute gradient $\bm{g} = \nabla_{\bm{w}} L^{Mix} (\bm{w} + \rho \frac{\nabla_{\bm{w}} L^{Mix}  (\bm{w})}{\|\nabla_{\bm{w}} L^{Mix}(\bm{w})\|_2})$.
		\STATE Update $\bm{w} = \bm{w} - \eta \bm{g}$ for solving $L^{GMix} (\bm{w}, \bm{\hat{z}})$ in Eq.~\eqref{Eq:gmix}.			
		\ENDFOR
	\end{algorithmic}
\end{algorithm}

\subsection{Theoretical Analysis}\label{Subsec:theo}

In literature, despite many empirical successes achieved by Mixup and SAM methods individually, how they theoretically improve the generalization ability of DNNs has not been well investigated and remains an open problem. As such, in this paper, we provide the theoretical analysis of the developed G-Mix learning framework. Note that as inspired by \cite{zhang2021how}, we analyze both the Mixup and SAM techniques from a regularization view of their corresponding training loss objectives, respectively. To present our analysis, we introduce the following assumption, lemmas, theorems, and proofs.    

\begin{assumption}\label{Assump:1}
	\emph{(Lipschitz Smooth).} Towards the given training dataset $\mathcal{S}$, the ERM loss function $L_{\mathcal{S}}$ is differentiable and the gradients $\nabla_{\bm{w}} L(\bm{w}, \bm{z})$ and $\nabla_{\bm{z}} L(\bm{w}, 
	\bm{z})$ are Lipschitz continuous with respect to $\bm{w}$ and $\bm{z}$ as
	\begin{equation}\label{Eq:assump}
		\begin{split}
			& \|\nabla_{\bm{w}} L(\bm{u}, \bm{z}) - \nabla_{\bm{w}} L(\bm{v}, \bm{z}) \| \leq \kappa_1 \|\bm{u} - \bm{v}\|, \nonumber \\
			& \|\nabla_{\bm{z}} L(\bm{w}, \bm{p}) - \nabla_{\bm{z}} L(\bm{w}, \bm{q}) \| \leq \kappa_2 \|\bm{p} - \bm{q}\|, \nonumber 
		\end{split}
	\end{equation}
	where $\bm{p}, \bm{q}$ are denoted as two arbitrary input training examples i.i.d drawn from $\mathcal{S}$ and $\bm{u}, \bm{v}$ represent two $\bm{w}$ states.
\end{assumption}

\begin{lemma}
	\emph{(Mixup Lipschitz Continuity.)} When the ERM loss is $\kappa_2$-Lipschitz continuous, the corresponding $L^{Mix}$ of the Mixup method satisfies
	\begin{equation}\label{Eq:lemma_mixup}
		\begin{split}
			& ||L^{Mix} (\bm{w}, \bm{\hat{z}}) - [\lambda L^{Mix} (\bm{w}, \bm{z}_i)   \\
			& \quad + (1-\lambda) L^{Mix} (\bm{w}, \bm{z}_j)  ] ||\leq C || \bm{z}_i - \bm{z}_j||^2,
		\end{split}
	\end{equation}
	where $C$ is a constant of $(\lambda, \bm{z}_i, \bm{z}_j)$ that satisfies $C = \frac{1}{n^2} \sum_{i,j =1}^{n} (\frac{|\lambda (\lambda-1)|}{2} \kappa_2 \|\bm{z}_i - \bm{z}_j \| )$. 
\end{lemma}

\begin{proof}
	Note that in order for this paper to be self-contained, we restate the proof from Proposition~1 in \cite{mai2021metamixup}. We start from the $L^{Mix}$ loss objective in Eq.~\eqref{Eq:mixup} that 
	\begin{equation}
		L^{Mix} (\bm{w}, \bm{\hat{z}})= \frac{1}{n^{2}} \sum_{i, j=1}^{n} \mathbb{E}_{\lambda \sim \mathcal{D}_\lambda} l( \bm{w}, (\lambda \bm{z}_i + (1-\lambda) \bm{z}_j )),
	\end{equation}
	we focus on the right side term $l( \bm{w}, (\lambda \bm{z}_i + (1-\lambda) \bm{z}_j ))$ that 
	\begin{equation}\label{Eq:lemma_proof_1}
		\begin{split}
			&l( \bm{w}, (\lambda \bm{z}_i + (1-\lambda) \bm{z}_j ))\\
			& = l( \bm{w}, \bm{z}_j) + \lambda \int_{0}^{1} \langle \nabla_{\bm{z}} l(\bm{w}, (\lambda s \bm{z}_i + (1-\lambda s) \bm{z}_j )), \bm{z}_i - \bm{z}_j \rangle ds     \\
			& = l( \bm{w}, \bm{z}_j) + \lambda( l( \bm{w}, \bm{z}_i) - l( \bm{w}, \bm{z}_j) ) \\
			& + \lambda (\int_{0}^{1} \langle \nabla_{\bm{z}} l(\bm{w}, (\lambda s \bm{z}_i + (1-\lambda s) \bm{z}_j )), \bm{z}_i - \bm{z}_j \rangle ds) \\
			&  -  \lambda ( l( \bm{w}, \bm{z}_i) - l( \bm{w}, \bm{z}_j) ), \nonumber
		\end{split}
	\end{equation}
	then, by adding the results of last three lines in Eq.~\eqref{Eq:lemma_proof_1} back to the left side of Eq.~\eqref{Eq:lemma_mixup} we have 
	\begin{equation}\label{Eq:lemma_proof_2}
		\begin{split}
			& L^{Mix} (\bm{w}, \bm{\hat{z}}) - [\lambda L^{Mix} (\bm{w}, \bm{z}_i)   + (1-\lambda) L^{Mix} (\bm{w}, \bm{z}_j)  ] \\
			& \leq  \frac{1}{n^{2}} \sum_{i, j=1}^{n} [  \lambda l( \bm{w}, \bm{z}_i) + (1- \lambda) l( \bm{w}, \bm{z}_j)  + C_1] \\
			& - ( \lambda \frac{1}{n} \sum_{i=1}^{n}   l( \bm{w}, \bm{z}_i) + (1- \lambda ) \frac{1}{n} \sum_{j=1}^{n}   l( \bm{w}, \bm{z}_j)), 
		\end{split}
	\end{equation}
	where $C_1 = \lambda (\int_{0}^{1} \langle \nabla_{\bm{z}} l(\bm{w}, (\lambda s \bm{z}_i + (1-\lambda s) \bm{z}_j )), \bm{z}_i - \bm{z}_j \rangle ds - ( l( \bm{w}, \bm{z}_i) - l( \bm{w}, \bm{z}_j) ))$, as $\bm{z}_j$ in this paper is obtained by uniformly permuting $\bm{z}_{i}$ that shares the same distribution. Therefore, we expand $C_1$ with $l( \bm{w}, \bm{z}_i) - l( \bm{w}, \bm{z}_j) = \int_{0}^{1} \langle  \nabla_{\bm{z}} l(\bm{w}, s\bm{z}_i + (1-s)\bm{z}_j), \bm{z}_i - \bm{z}_j \rangle ds$ as illustrated in \cite{mai2021metamixup} and can obtain
	\begin{equation}\label{Eq:lemma_proof_3}
		\begin{split}
			C_1 & = \lambda \int_{0}^{1} \langle \nabla_{\bm{z}} l(\bm{w}, (\lambda s \bm{z}_i + (1-\lambda s) \bm{z}_j )) \\
			&- \nabla_{\bm{z}} l(\bm{w}, s\bm{z}_i + (1-s)\bm{z}_j), \bm{z}_i - \bm{z}_j \rangle ds \\
			& \overset{(a)}{\leq} \lambda \int_{0}^{1} \|\nabla_{\bm{z}} l(\bm{w}, (\lambda s \bm{z}_i + (1-\lambda s) \bm{z}_j )) \\
			& - \nabla_{\bm{z}} l(\bm{w}, s\bm{z}_i + (1-s)\bm{z}_j) \| \|\bm{z}_i - \bm{z}_j \| ds \\
			& \overset{(b)}{\leq} \lambda \int_{0}^{1} \kappa_2  \| (\lambda-1) s\bm{z}_i + (1-\lambda) s\bm{z}_j \| \|\bm{z}_i - \bm{z}_j \|  ds \\
			& = \lambda \int_{0}^{1} \kappa_2 |(\lambda-1)| s \|\bm{z}_i - \bm{z}_j \| ^{2} ds \\
			& = \frac{|\lambda (\lambda-1)|}{2} \kappa_2 \|\bm{z}_i - \bm{z}_j \|^{2}, 
		\end{split}
	\end{equation}
	where $(a)$ follows the Cauchy-Schwartz inequality, and $(b)$ satisfies the Assumption.~\ref{Assump:1}. Then, by adding up the results in Eq.~\eqref{Eq:lemma_proof_3} to the left side of Eq.~\eqref{Eq:lemma_mixup},  the proof is done.
\end{proof}

\begin{theorem}
	\emph{(G-Mix framework Lipschitz Continuity.)} The gradient $\nabla L^{GMix} (\bm{w} + \bdelta, \bm{\hat{z}})$ of the developed G-Mix framework with respect to $\bm{w}$ as $\nabla_{\bm{w}} L^{Mix} (\bm{w} + \bdelta, \bm{\hat{z}})$ is Lipschitz continuous with constant $2\kappa_{1}+ 2\kappa_{1}^{2}\rho$, and to $\bm{z}$ as $\nabla_{\bm{z}} L^{Mix} (\bm{w} + \bdelta, \bm{\hat{z}})$ with Lipschitz constant $2C$.
\end{theorem}

\begin{proof}
	We start from the two group of status towards the loss objective in $L^{GMix}$ with respect to different SAM models and Mixup data as $(\bm{u}+ \bdelta, \bm{p})$ and $(\bm{v} + \bdelta, \bm{q})$ that
	\begin{equation}
		\begin{split}
			& \mathbb{E} || \nabla L^{GMix} L(\bm{u}+ \bdelta, \bm{p}) - \nabla L^{GMix} (\bm{v} + \bdelta, \bm{q})||\\ 
			& = \mathbb{E} || \nabla L^{GMix} L(\bm{u}+ \bdelta, \bm{p}) -  \nabla L^{GMix} L(\bm{v}+ \bdelta, \bm{p})    \\
			& \quad + \nabla L^{GMix} L(\bm{v}+ \bdelta, \bm{p}) - \nabla L^{GMix} (\bm{v} + \bdelta, \bm{q})|| \\
			& \overset{(a)}{\leq} 2 \mathbb{E} [|| \nabla L^{GMix} L(\bm{u}+ \bdelta, \bm{p}) -  \nabla L^{GMix} L(\bm{v}+ \bdelta, \bm{p})|| \\
			& \quad + ||\nabla L^{GMix} L(\bm{v}+ \bdelta, \bm{p}) - \nabla L^{GMix} (\bm{v} + \bdelta, \bm{q})||] \\
			&  = 2 \underbrace{\| \nabla_{\bm{w}} L^{Mix} (\bm{u} + \bdelta, \bm{p}) - \nabla_{\bm{w}} L^{Mix} (\bm{v} + \bdelta, \bm{p}) \|}_{A} \\
			& \quad +  2 \underbrace{\|(\nabla_{\bm{z}} L^{Mix} (\bm{v} + \bdelta, \bm{p}) -   \nabla_{\bm{z}} L^{Mix} (\bm{v} + \bdelta, \bm{q}) ) \|}_{B}, \\
		\end{split}
	\end{equation}
	where $(a)$ is due to the fact that $\mathbb{E}[\|x_1 + \dots + x_n \|] \leq n \mathbb{E}[\|x_1\| + \dots + x_n]$. And for the term $A$, inspired by 
	\cite{zhao2022penalizing} we can have 
	\begin{equation}
		\begin{split}
			A & = \| \nabla_{\bm{w}} L^{Mix} (\bm{u} + \rho \frac{\nabla_{\bm{w}} L^{Mix} (\bm{u}, \bm{p})}{\| \nabla_{\bm{w}} L^{Mix} (\bm{u}, \bm{p}) \|}) \\
			& \quad -  \nabla_{\bm{w}} L^{Mix} (\bm{v} + \rho \frac{\nabla_{\bm{w}} L^{Mix} (\bm{v}, \bm{p})}{\| \nabla_{\bm{w}} L^{Mix} (\bm{v}, \bm{p}) \|}) \| \\
			& \overset{(b)}{\leq} \kappa_1 \| (\bm{u} + \rho \frac{\nabla_{\bm{w}} L^{Mix} (\bm{u}, \bm{p})}{\| \nabla_{\bm{w}} L^{Mix} (\bm{u}, \bm{p}) \|}) - (\bm{v} + \rho \frac{\nabla_{\bm{w}} L^{Mix} (\bm{v}, \bm{p})}{\| \nabla_{\bm{w}} L^{Mix} (\bm{v}, \bm{p}) \|}) \|  \\
			& \leq \kappa_1 \|\bm{u} - \bm{v} \| +\kappa_1 \rho \| \frac{\nabla_{\bm{w}} L^{Mix} (\bm{u}, \bm{p})}{\| \nabla_{\bm{w}} L^{Mix} (\bm{u}, \bm{p}) \|} - \frac{\nabla_{\bm{w}} L^{Mix} (\bm{v}, \bm{p})}{\| \nabla_{\bm{w}} L^{Mix} (\bm{v}, \bm{p}) \|} \|  \\
			& \leq (\kappa_1 + \kappa_{1}^{2} \rho) \| \bm{u} - \bm{v} \|,\\ \nonumber
		\end{split}
	\end{equation}
	and for the term $B$, we could easily draw the answer from the obtained Lemma.~1 above. Then the proofs are done. 
\end{proof}

\section{BG-Mix and DG-Mix Algorithms}\label{Sec:algorithm}
In this section, we propose two novel algorithms, called Binary G-Mix (BG-Mix) and Decomposed G-Mix (DG-Mix), which aim to address the ``manifold intrusion" problem in Mixup \cite{guo2019mixup} within the developed G-Mix learning framework.

The ``manifold intrusion" problem refers to conflicts between the synthetic labels of the generated samples and the original labels for a given training set $\mathcal{S}$. To mitigate this problem, we analyze it from the perspective of finding ``flat minima" in the SAM procedure within the G-Mix framework. The $\bdelta$ approximation procedure in the G-Mix framework introduces a quantified loss change between $L^{Mix}(\bm{w} + \bdelta, \hat{\bm{z}})$ and $L^{Mix}(\bm{w}, \bm{\hat{z}})$. Intuitively, we consider that training samples with larger loss changes due to weight perturbation are more sensitive to finding ``flat minima," resulting in lower probabilities of causing the "manifold intrusion" problem.

To address this issue, we first develop a sharpness-sensitive selection process for the quantified loss change $L^{Mix}(\bm{w} + \bdelta, \hat{\bm{z}}) - L^{Mix}(\bm{w}, \bm{\hat{z}})$ for a given set of training examples. Then, the BG-Mix and DG-Mix algorithms apply different methods to further enhance the performance of the G-Mix framework. The BG-Mix algorithm focuses on improving the performance by disregarding the less-sensitive examples during the learning process. By excluding the less-sensitive examples, BG-Mix addresses the "manifold intrusion" problem and improves the computational efficiency of the G-Mix framework.

On the other hand, the DG-Mix algorithm decomposes the loss of less-sensitive examples into two components: one component is parallel to the averaged direction of the sensitive examples, and the other component is orthogonal to that direction. DG-Mix performs an extra back-forward operation on the orthogonal component as a regularization term, further improving the generalization ability of DNNs. Note that both the BG-Mix and DG-Mix algorithms aim to enhance the performance of the G-Mix framework by effectively handling the ``manifold intrusion" problem. Their specific approaches are detailed as follows.

\subsection{Sharpness-sensitive Selection}
Considering a training batch $\mathcal{B} = \{\bm{\hat{z}}_i\}$ after the Mixup procedure in G-Mix, we can divide the data samples in $\mathcal{B}$ into two subsets by evaluating their quantified sharpness-sensitivity:
\begin{equation}\label{Eq:two_set}
	\begin{split}
		\mathcal{B}^{+}: &= \{ \bm{\hat{z}}_i \subseteq \mathcal{B} | l(\bm{w} + \bdelta, \bm{\hat{z}}_i) - l(\bm{w}, \bm{\hat{z}}_i) \geq \xi \} \\
		\mathcal{B}^{-}: &= \{ \bm{\hat{z}}_i \subseteq \mathcal{B} | l(\bm{w} + \bdelta, \bm{\hat{z}}_i) - l(\bm{w}, \bm{\hat{z}}_i)  < \xi \}, \\
	\end{split}
\end{equation}
where $\mathcal{B}^{+}$ could be considered more sharpness-sensitive than $\mathcal{B}^{-}$ and the added hyper-parameter $\xi$ is an empirical threshold that can be determined by changing the size ratio $\gamma$ where $\gamma = \frac{|\mathcal{B}^{+}|}{|\mathcal{B}|}$. Additionally, we further provide the theoretical analysis of the generalization ability relationship between $\mathcal{B}^{+}$ and $\mathcal{B}$ as the following theorem.

\begin{theorem}
	\emph{(Upper-bounded Loss \cite{du2022efficient}.)} For a fixed value of hyper-parameter $\gamma \in [0,1]$, the loss of $L^{GMix}_{\mathcal{B}} (\bm{w}+ \bdelta, \bm{\hat{z}}) $ can be regarded as upper-bounded by $L^{GMix}_{\mathcal{B}^{+}} (\bm{w}+ \bdelta, \bm{\hat{z}})$. 
\end{theorem}

\begin{proof}
	We start by divide the $L^{GMix}_{\mathcal{B}} (\bm{w}+ \bdelta, \bm{\hat{z}}) $ into two components towards the obtained sharpness-sensitivity as 
	\begin{equation}\label{Eq:theorem_2_1}
		\resizebox{0.99\columnwidth}{!}{
			$
			\begin{split}
				& L^{GMix}_{\mathcal{B}} (\bm{w}+ \bdelta, \bm{\hat{z}}) =  \gamma L^{GMix}_{\mathcal{B}^{+}} (\bm{w}+ \bdelta, \bm{\hat{z}}) + (1- \gamma)L^{GMix}_{\mathcal{B}^{-}} (\bm{w}+ \bdelta, \bm{\hat{z}}) \\
				&=   L^{GMix}_{\mathcal{B}^{+}} (\bm{w}+ \bdelta, \bm{\hat{z}}) + (1- \gamma) (L^{GMix}_{\mathcal{B}^{-}} (\bm{w}+ \bdelta, \bm{\hat{z}}) - L^{GMix}_{\mathcal{B}^{+}} (\bm{w}+ \bdelta, \bm{\hat{z}})) \\
				& \overset{(a)}{\leq} L^{GMix}_{\mathcal{B}^{+}} (\bm{w}+ \bdelta, \bm{\hat{z}}) + (1 - \gamma)[R^{GMix}_{\mathcal{B}^{-}} (\bm{w}, \bm{\hat{z}}) + L^{Mix}_{\mathcal{B}^{-}}(\bm{w} + \bdelta,\bm{\hat{z}}) \\& \quad - R^{GMix}_{\mathcal{B}^{+}} (\bm{w}, \bm{\hat{z}}) + L^{Mix}_{\mathcal{B}^{+}}(\bm{w} + \bdelta,\bm{\hat{z}} )], \nonumber		
			\end{split}
			$}
	\end{equation}
	where $(a)$ can be deviated from the definition in Eq.~\eqref{Eq:loss_upperbound} that $R^{GMix} (\bm{w}, \bm{\hat{z}}) = \max_{\bdelta: \|\bdelta \|_2 \leq \rho} L^{Mix} (\bm{w}+\bdelta, \bm{\hat{z}} ) - L^{Mix} (\bm{w}, \bm{\hat{z}})$. For the given training batch $\mathcal{B}$, we can further expand $R^{GMix}$
	\begin{equation}\label{Eq:heorem_2_r}
		R^{GMix}_{\mathcal{B}}(\bm{w}, \bm{\hat{z}}) = \frac{1}{|\mathcal{B}|} \sum_{i}^{|\mathcal{B}|} \mathbb{E} [l( \bm{w}+ \bdelta, \bm{\hat{z}}_{i})- l( \bm{w}, \bm{\hat{z}}_{i})].
	\end{equation}
	Thus, with a sorted $l(\bm{w} + \bdelta, \bm{\hat{z}}_i) - l(\bm{w}, \bm{\hat{z}}_i)$ as illustrated in Eq.~\eqref{Eq:two_set}, we will have $	R^{GMix}_{\mathcal{B}^{-}}(\bm{w}, \bm{\hat{z}}) \leq R^{GMix}_{\mathcal{B}^{+}}(\bm{w}, \bm{\hat{z}})$. Additionally, we can also apply the Taylor approximation on $l( \bm{w}+ \bdelta, \bm{\hat{z}}_{i})$ from the $L^{Mix}_{\mathcal{B}} (\bm{w} + \bdelta,\bm{\hat{z}}) = \frac{1}{|\mathcal{B}|} \sum_{i}^{|\mathcal{B}|} \mathbb{E} [l( \bm{w}+ \bdelta, \bm{\hat{z}}_{i})]$ around $\bm{w}$ as 
	\begin{equation}\label{Eq:heorem_2_taylor}
		l( \bm{w}+ \bdelta, \bm{\hat{z}}_{i}) = l( \bm{w}, \bm{\hat{z}}_{i}) + \|\bdelta \| \nabla_{\bm{w}} l( \bm{w}, \bm{\hat{z}}_{i}) + \mathcal{O}(\|\bdelta \|),
	\end{equation} 
	where $\mathcal{O}(\|\bdelta \|)$ represents higher terms in the Taylor expansion that can be omitted here in our work. Then, Eq.~\eqref{Eq:heorem_2_taylor} indicates that the result of $l( \bm{w}+ \bdelta, \bm{\hat{z}}_{i}) - l( \bm{w}, \bm{\hat{z}}_{i})$ is positively correlated to $\nabla_{\bm{w}} l( \bm{w}, \bm{\hat{z}}_{i})$. As demonstrated in \cite{li2019gradient,du2022efficient}, the difficult training samples in DL (e.g., the examples with higher loss) usually generate gradients with larger magnitudes. Therefore, it is reasonable to assume that with a sorted $l( \bm{w}+ \bdelta, \bm{\hat{z}}_{i}) - l( \bm{w}, \bm{\hat{z}}_{i})$ as demonstrated in Eq.~\eqref{Eq:two_set}, the corresponding $l( \bm{w}, \bm{\hat{z}}_{i})$ is with high possibility of being sorted, which leads to the claim that 
	\begin{equation}\label{Eq:subset_l}
		L^{Mix}_{\mathcal{B}^{-}}(\bm{w} + \bdelta,\bm{\hat{z}} ) - L^{Mix}_{\mathcal{B}^{+}}(\bm{w} + \bdelta,\bm{\hat{z}} )  \leq 0.
	\end{equation}
	By adding back the results to Eq.~\eqref{Eq:theorem_2_1}, the proof done.
\end{proof}

\subsection{Binary G-Mix and Decomposed G-Mix}
Then, we apply two different approaches to further improve the performance of DNNs on G-Mix with the obtained two subsets $\mathcal{B}^{+}$ and $\mathcal{B}^{-}$ for the sharpness-sensitivity. Firstly, one natural solution is to simply ignore the training examples in $\mathcal{B}^{-}$, which can both reduce the chances of ``manifold intrusion" and increase the efficiency of G-Mix. Under this condition, the training function for the given batch $\mathcal{B}$ of G-Mix would turn to 
\begin{equation}\label{Eq:bgmix}
	L^{BG-Mix} (\bm{w},\bm{\hat{z}}) =   \frac{1}{|\mathcal{B}^{+}|} \sum_{i \in \mathcal{B}^{+}} l( \bm{w} + \bdelta, \bm{\hat{z}_{i}}), 
\end{equation}
where $L^{BG-Mix}=  L^{GMix}_{\mathcal{B}^{+}}$, and we call this approach as $\underline{B}$inary G-Mix. Secondly, we consider that although the training examples in $\mathcal{B}^{-}$ can be less important, they may still capture useful training information that can improve the learning performance of the DNN model. As such, by denoting the averaged loss $\bar{\bm{l}} = \frac{1}{|\mathcal{B}^{+}|} \sum_{i \in \mathcal{B}^{+}} l( \bm{w} + \bdelta, \bm{\hat{z}_{i}})$, we decompose the training loss of $l( \bm{w} + \bdelta, \bm{\hat{z}_{i}})$ in $\mathcal{B}^{-}$ into two components that are receptively parallel and orthogonal to the direction of $\bar{\bm{l}}$ that could be formalized as 
\begin{equation}\label{Eq:dgmix_decompose}
	l( \bm{w} + \bdelta, \bm{\hat{z}_{i,j}}) = l_{\parallel}( \bm{w} + \bdelta, \bm{\hat{z}_{i}}) + l_{\perp}( \bm{w} + \bdelta, \bm{\hat{z}_{i}}), \text{for $i \in \mathcal{B}^{-}$}.
\end{equation}
Then, by adding the orthogonal part $l_{\perp}( \bm{w} + \bdelta, \bm{\hat{z}_{i}})$ to the loss objective, we can both search the ``flat minima" in the SAM procedure of G-Mix via minimizing $\mathcal{B}^{+}$ and protect the potentially useful information in $\mathcal{B}^{-}$ at the same time, which is called the $\underline{D}$ecomposed G-Mix approach 
\begin{equation}\label{Eq:dgmix}
	L^{DG-Mix} (\bm{w},\bm{\hat{z}}) =   L^{BG-Mix}(\bm{w},\bm{\hat{z}}) + \frac{1}{|\mathcal{B}^{-}|} \sum_{i \in \mathcal{B}^{-}} l_{\perp} ( \bm{w} + \bdelta, \bm{\hat{z}_{i}}). \nonumber
\end{equation}
We illustrate the proposed BG- and DG-Mix algorithms at the Algorithm.~\ref{alg:algorithms}. Note that for better comparison and presentation, we highlight the major difference between the developed G-Mix framework, the proposed BG-Mix, and the proposed DG-Mix algorithms with different colors in Algorithm.~\ref{alg:algorithms}. 

\begin{algorithm}[tb]
	\caption{\colorbox[rgb]{0.74,0.83,1}{G-Mix}, \colorbox[rgb]{0.5,0.9,0.5}{BG-Mix} and \colorbox[rgb]{1.0, 0.55, 0.41}{DG-Mix}}
	\label{alg:algorithms}
	\begin{algorithmic}[1]
		\STATE {\bfseries Input:}  Batch $\mathcal{B} = \{\bm{\hat{z}}_i \}$, $\gamma$ (the threshold controller).
		\STATE $\underline{\text{Sharpness-sensitive selection:}}$
		\STATE {With $\gamma = \frac{|\mathcal{B}^{+}|}{|\mathcal{B}|}$, the empirical threshold $\xi$ is determined}
		\STATE {Obtain two subsets $\mathcal{B}^{+}$ and $\mathcal{B}^{-}$ by sorting $l(\bm{w} + \bdelta, \bm{\hat{z}}_i) - l(\bm{w}, \bm{\hat{z}}_i)$ towards the threshold $\xi$ that \begin{itemize}
				\item[] {$\mathcal{B}^{+}: = \{ \bm{\hat{z}}_i \subseteq \mathcal{B} | l(\bm{w} + \bdelta, \bm{\hat{z}}_i) - l(\bm{w}, \bm{\hat{z}}_i) \geq \xi \}$.}
				\item[] {$\mathcal{B}^{-}: = \{ \bm{\hat{z}}_i \subseteq \mathcal{B} | l(\bm{w} + \bdelta, \bm{\hat{z}}_i) - l(\bm{w}, \bm{\hat{z}}_i)  < \xi \}$.}
		\end{itemize}} 
		\STATE $\underline{\text{Loss Objectives:}}$
		\begin{itemize}
			\item[] \colorbox[rgb]{0.74,0.83,1}{$L^{G-Mix}  :   \frac{1}{|\mathcal{B}|} \sum_{i \in \mathcal{B}} l( \bm{w} + \bdelta, \bm{\hat{z}_{i}})$.}
			\item[] \colorbox[rgb]{0.5,0.9,0.5}{$L^{BG-Mix}  :   \frac{1}{|\mathcal{B}^{+}|} \sum_{i \in \mathcal{B}^{+}} l( \bm{w} + \bdelta, \bm{\hat{z}_{i}})$.}
			\item[] \colorbox[rgb]{1.0, 0.55, 0.41}{$L^{DG-Mix} :   L^{BG-Mix} +  \frac{1}{|\mathcal{B}^{-}|} \sum_{i \in \mathcal{B}^{-}} l_{\perp} ( \bm{w} + \bdelta, \bm{\hat{z}_{i}})$.}
		\end{itemize}
		% 		, short buffer $\mathcal{S}$.
	\end{algorithmic}
\end{algorithm}

\subsection{Discussion}
In this paper, we make several contributions to improving the performance and generalization ability of DNN models. Firstly, we introduce the G-Mix learning framework, which combines the advantages of Mixup and SAM methods. Theoretical analysis in Section \ref{Subsec:theo} demonstrates that the G-Mix framework achieves Lipschitz continuity to both the weight parameters $\bm{w}$ and the input data $\bm{z}$. This analysis provides a solid foundation for understanding how the G-Mix framework improves generalization ability.

To further enhance the performance within the G-Mix framework, we propose two novel algorithms, namely BG-Mix and DG-Mix. These algorithms employ a sharpness-sensitive selection strategy that divides the training batch after the Mixup procedure into two subsets, denoted as $\mathcal{B}^{+}$ and $\mathcal{B}^{-}$. The BG-Mix algorithm discards the less-sensitive training examples in $\mathcal{B}^{-}$, which improves the efficiency of G-Mix and mitigates the ``manifold intrusion" risk associated with Mixup. On the other hand, the DG-Mix algorithm leverages the training information from $\mathcal{B}^{-}$ by decomposing the corresponding loss into parallel and orthogonal components relative to the averaged loss direction in $\mathcal{B}^{+}$. By incorporating the orthogonal component into the loss objective, the DG-Mix algorithm achieves better generalization than simply minimizing $\mathcal{B}^{+}$, while still reaching the same convergence as G-Mix.

It is worth noting that there exists a similar proposal called Efficient Sharpness-Aware Minimization (ESAM) \cite{du2022efficient}, which improves the efficiency of SAM by determining the most sharpness-sensitive examples. However, the key difference between BG-Mix and ESAM lies in their design motivations. BG-Mix specifically addresses the ``manifold intrusion" problem resulting from the Mixup procedure in G-Mix, whereas ESAM serves a different purpose. Additionally, while the DG-Mix algorithm incurs extra computation due to the loss decomposition process, this additional cost is negligible compared to the doubled back-propagation operations involved in SAM implementation.

Furthermore, in Section \ref{Sec:analysis}, we provide a detailed discussion on the computation cost of the compared approaches in this paper, addressing any concerns regarding the additional computational overhead introduced by the proposed algorithms. Overall, our contributions lie in the development of the G-Mix learning framework and the introduction of the BG-Mix and DG-Mix algorithms, which effectively improve the performance and generalization ability of DNN models while addressing specific challenges such as the ``manifold intrusion" problem.

\section{Experiments and Result Analysis}\label{Sec:analysis}

\begin{figure*}[tb]
	\centering
	\begin{subfigure}{0.19\columnwidth}
		\includegraphics[width = 1\columnwidth]{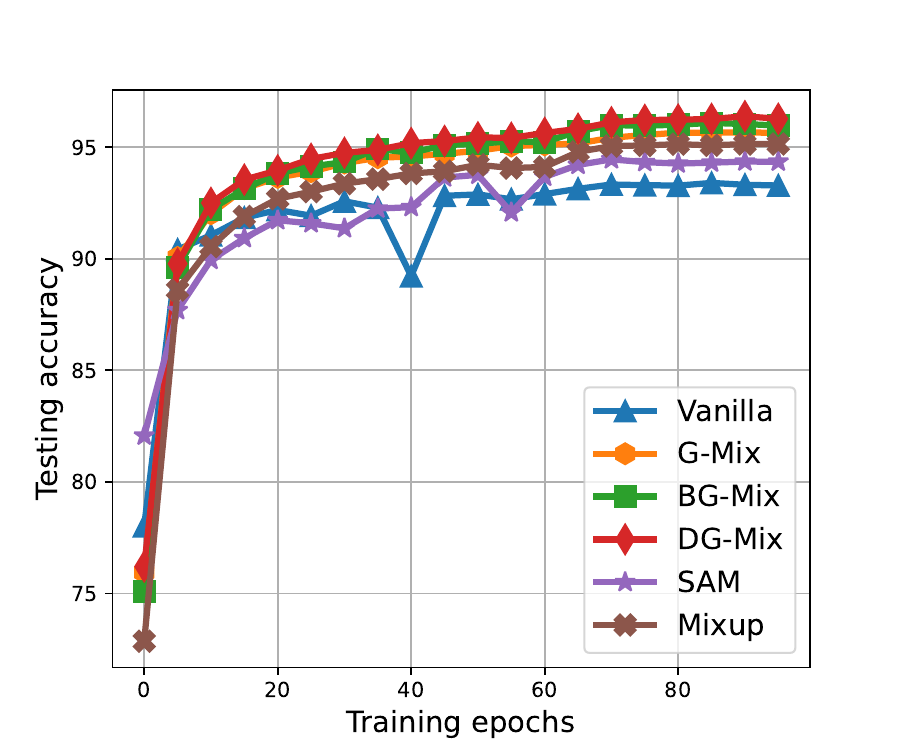}
		\caption{SVHN}
		\label{fig:convergence_svhn}
	\end{subfigure}
	\begin{subfigure}{0.19\columnwidth}
		\includegraphics[width = 1\columnwidth]{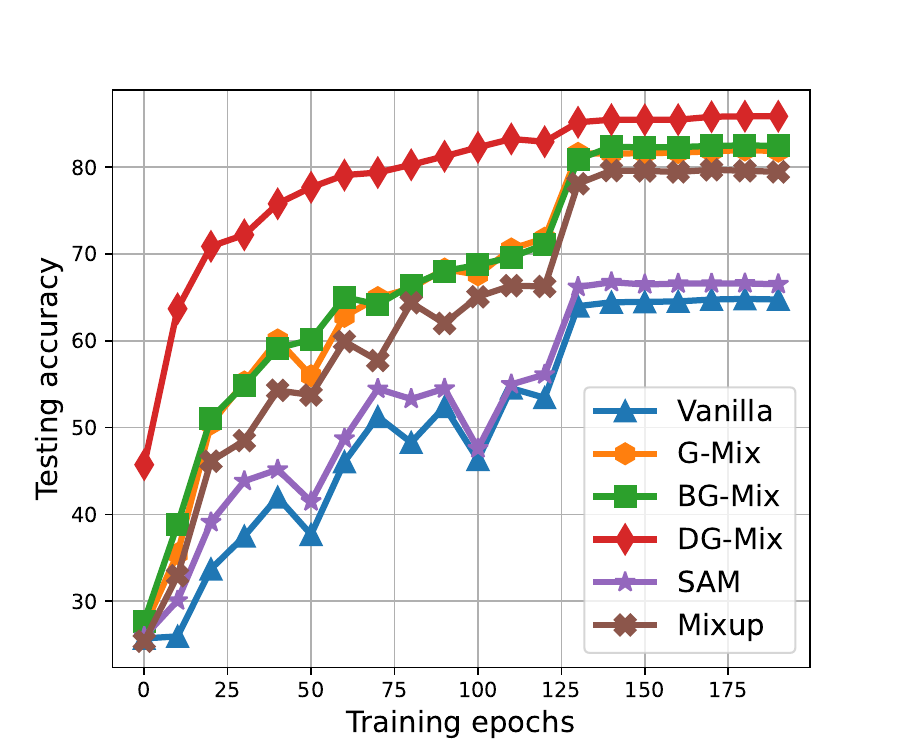}
		\caption{STL-10}
		\label{fig:convergence_cifar10}
	\end{subfigure}
	\begin{subfigure}{0.19\columnwidth}
		\includegraphics[width = 1\columnwidth]{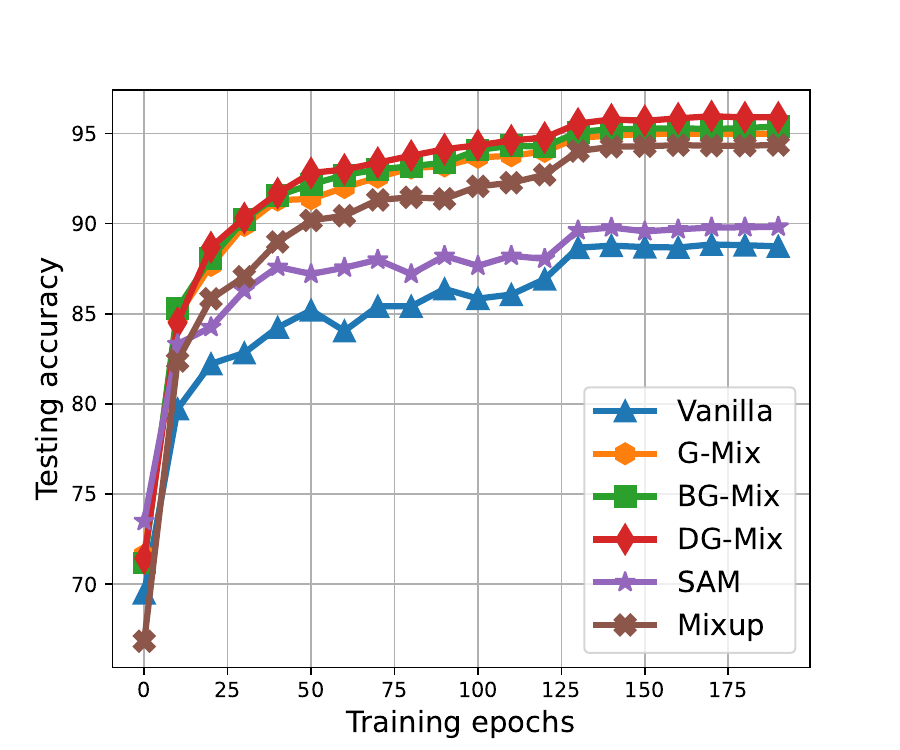}
		\caption{Cifar-10}
		\label{fig:convergence_stl10}
	\end{subfigure}
	\begin{subfigure}{0.19\columnwidth}
		\includegraphics[width = 1\columnwidth]{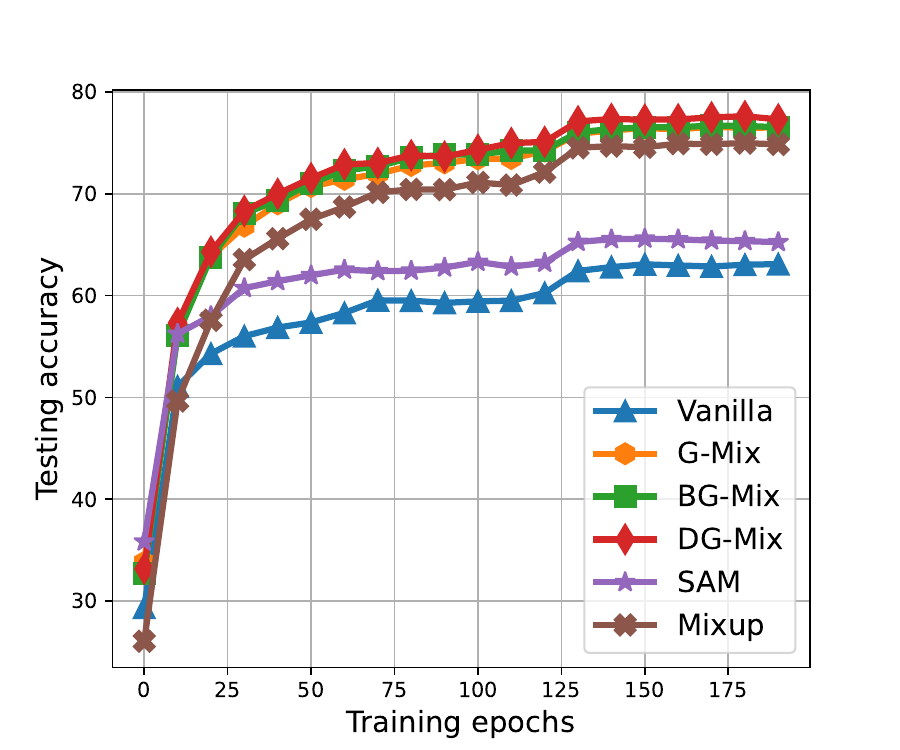}
		\caption{Cifar-100}
		\label{fig:convergence_cifar100}
	\end{subfigure}
	\begin{subfigure}{0.19\columnwidth}
		\includegraphics[width = 1\columnwidth]{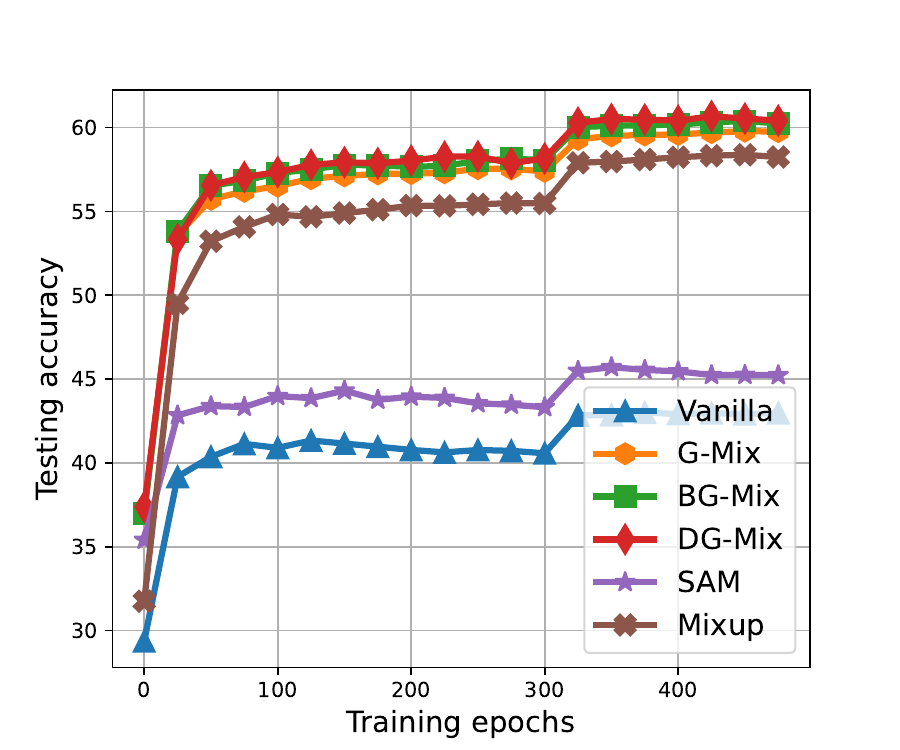}
		\caption{Tiny-ImageNet}
		\label{fig:convergence_tinyimagenet}
	\end{subfigure}
	%	\caption{text}
	\caption{Learning convergence of testing accuracy for the compared methods against multiple datasets.}
	\label{fig:convergence}
\end{figure*}

\begin{figure*}[tb]
	\centering
	\begin{subfigure}{0.19\columnwidth}
		\includegraphics[width = 1\columnwidth]{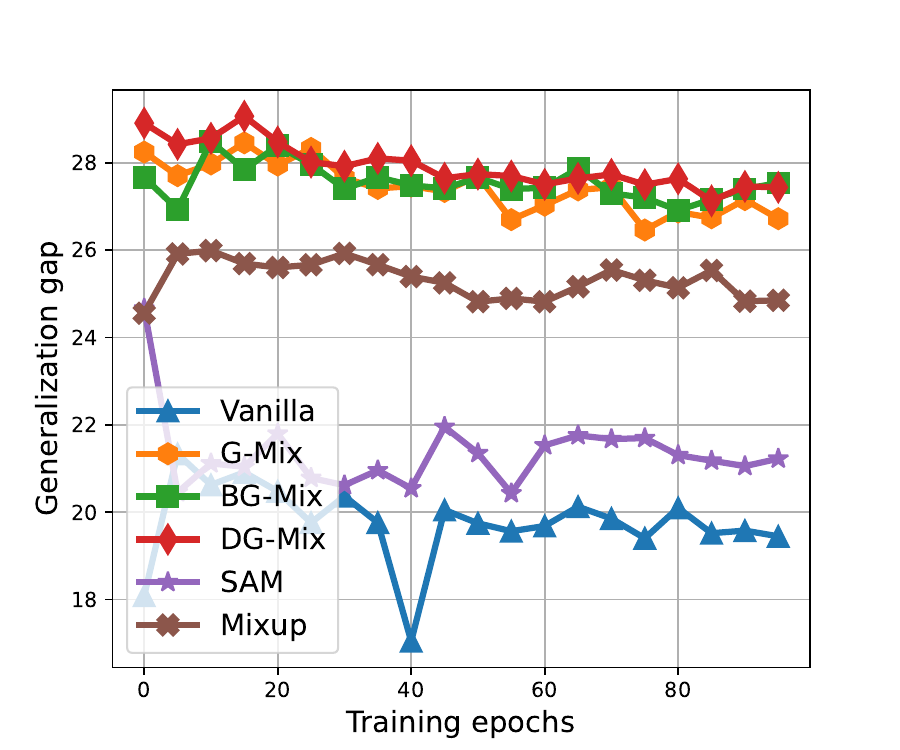}
		\caption{SVHN}
		\label{fig:gap_svhn}
	\end{subfigure}
	\begin{subfigure}{0.19\columnwidth}
		\includegraphics[width = 1\columnwidth]{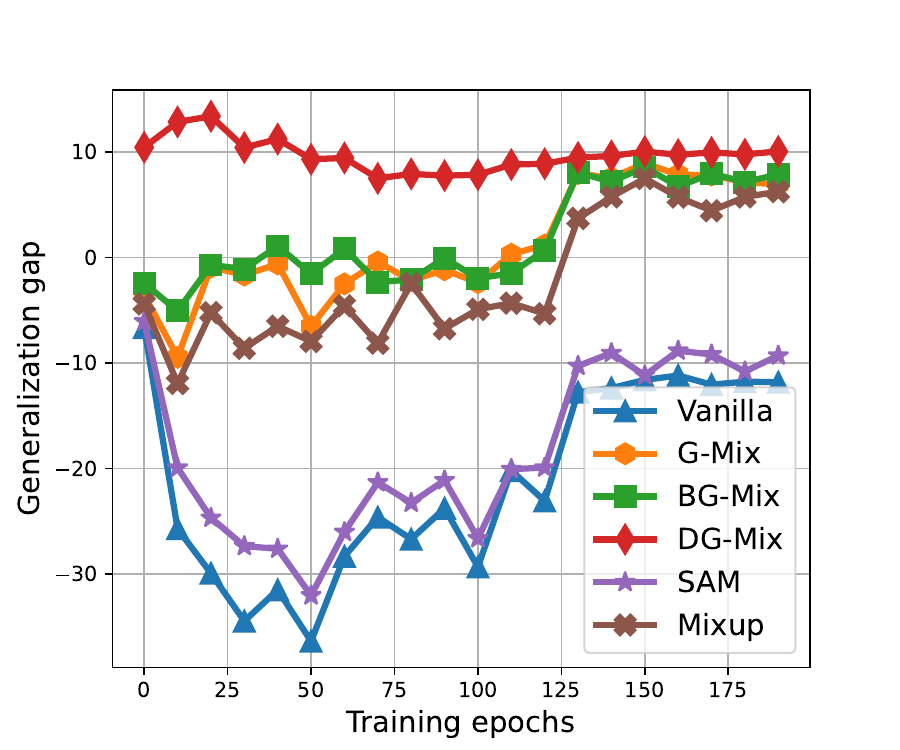}
		\caption{STL-10}
		\label{fig:gap_cifar10}
	\end{subfigure}
	\begin{subfigure}{0.19\columnwidth}
		\includegraphics[width = 1\columnwidth]{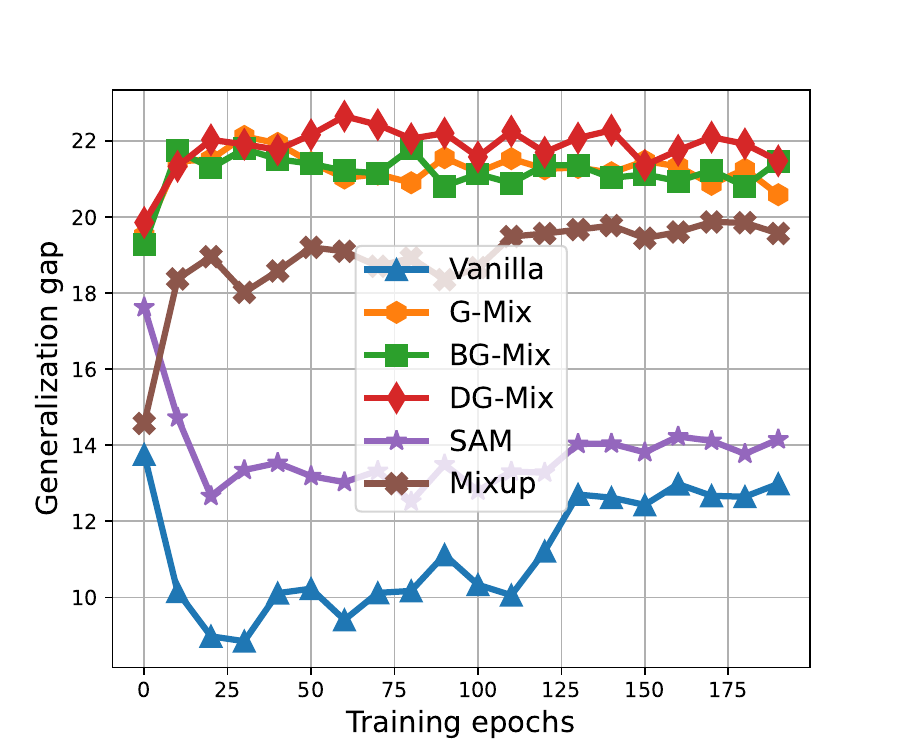}
		\caption{Cifar-10}
		\label{fig:gap_stl10}
	\end{subfigure}
	\begin{subfigure}{0.19\columnwidth}
		\includegraphics[width = 1\columnwidth]{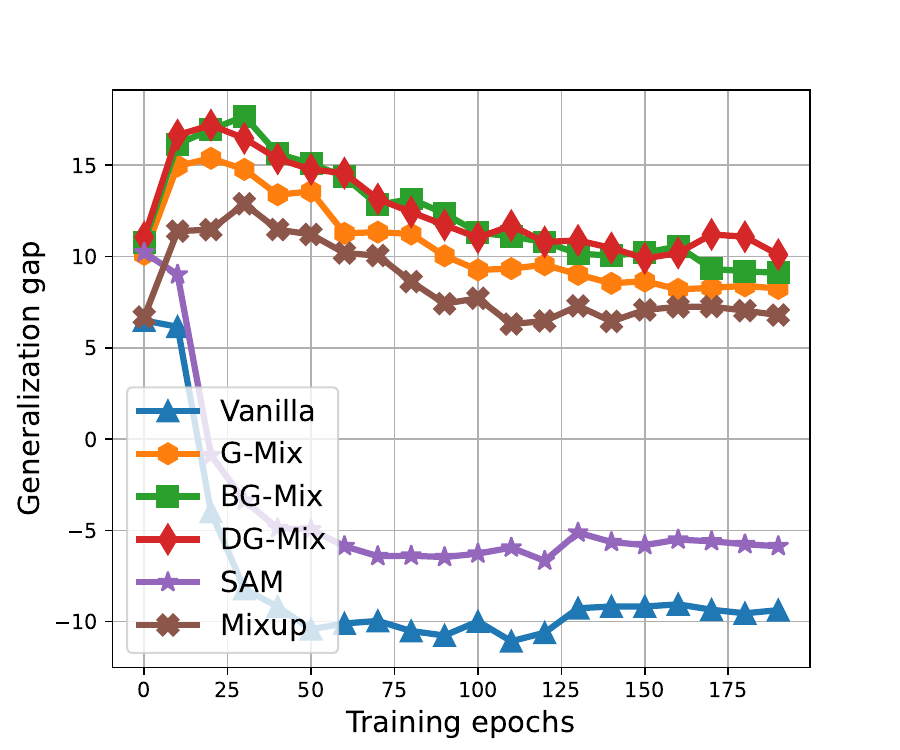}
		\caption{Cifar-100}
		\label{fig:gap_cifar100}
	\end{subfigure}
	\begin{subfigure}{0.19\columnwidth}
		\includegraphics[width = 1\columnwidth]{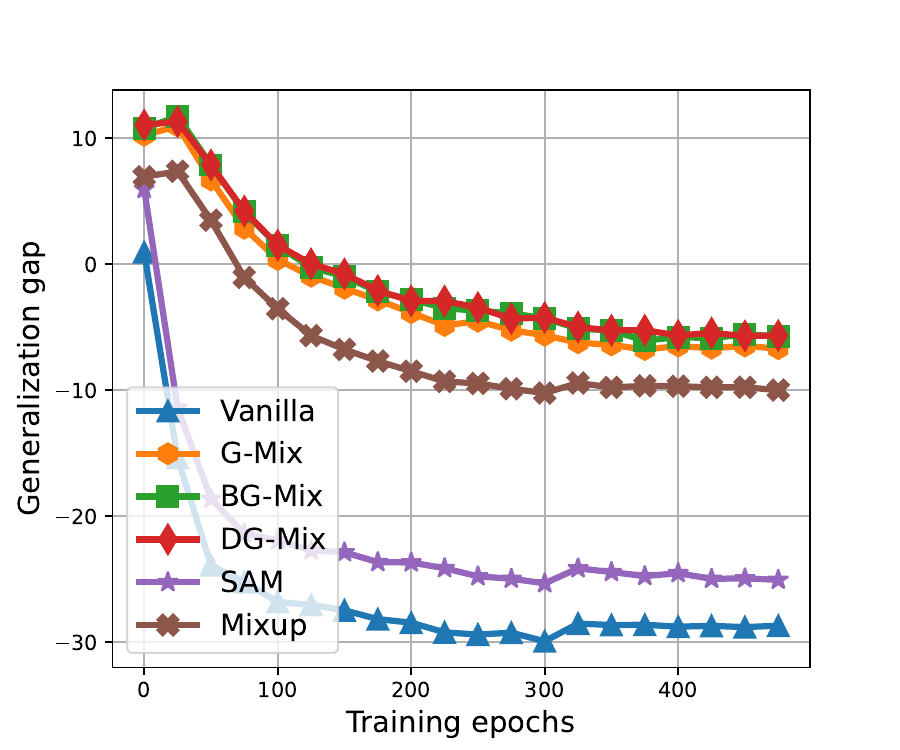}
		\caption{Tiny-ImageNet}
		\label{fig:gap_tinyimagenet}
	\end{subfigure}
	%	\caption{text}
	\caption{Learning convergence of generalization gap for the compared methods against multiple datasets.}
	\label{fig:gap}
\end{figure*}

\begin{figure*}[tb]
	\centering
	\begin{subfigure}{0.24\columnwidth}
		\includegraphics[width = 1\columnwidth]{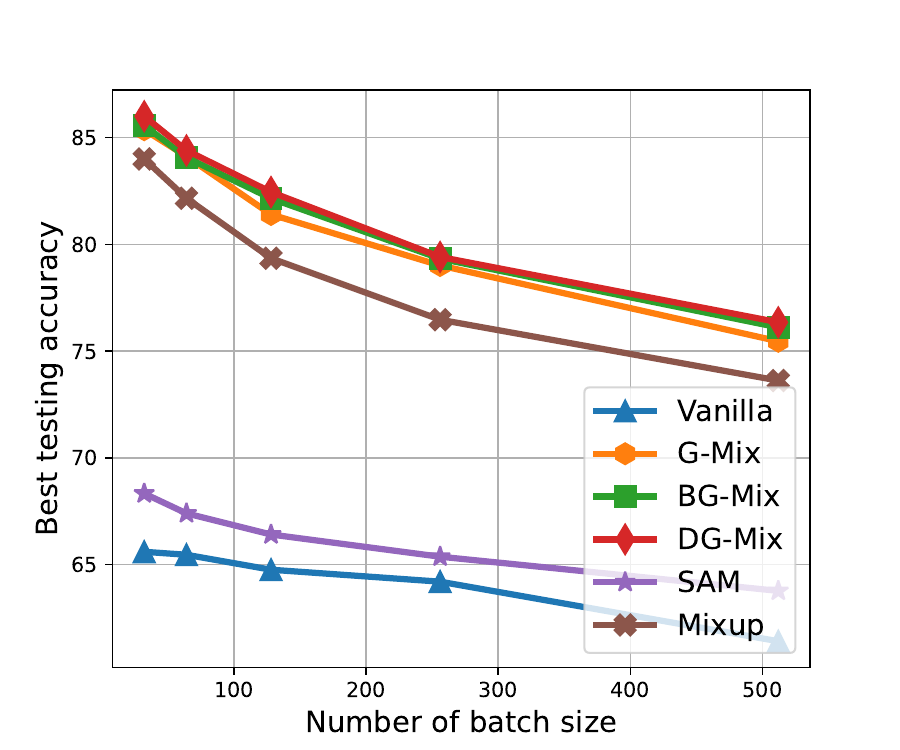}
		\caption{STL-10: ${\bf{Test}\textsubscript{Acc}}$}
		\label{fig:batch_acc_stl10}
	\end{subfigure}
	\begin{subfigure}{0.24\columnwidth}
		\includegraphics[width = 1\columnwidth]{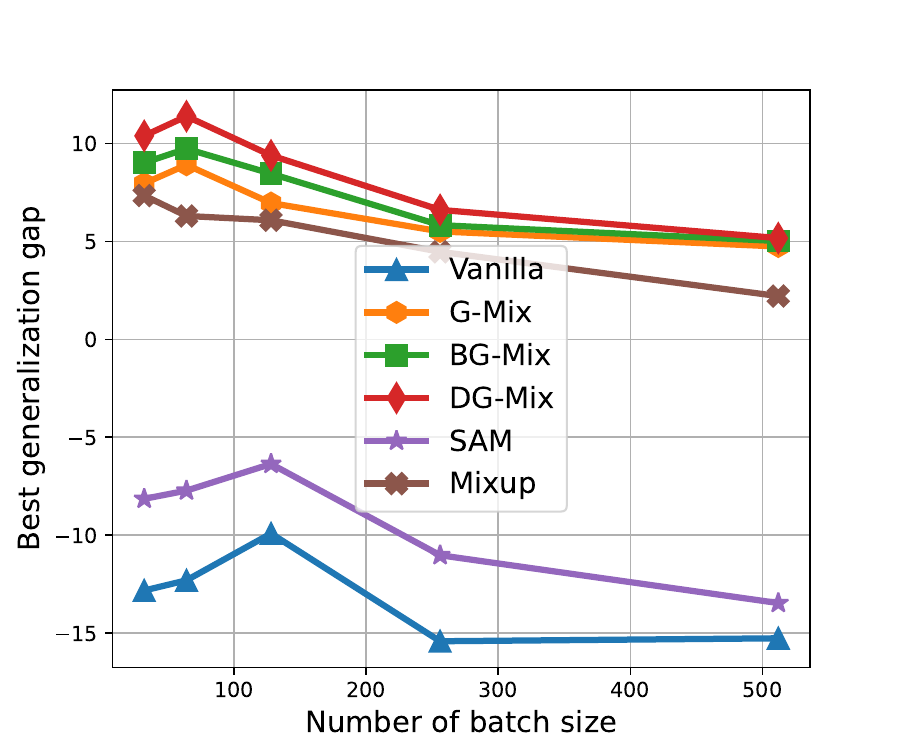}
		\caption{STL-10: ${\bf{Gap}}$}
		\label{fig:batch_gap_stl10}
	\end{subfigure}
	\begin{subfigure}{0.24\columnwidth}
		\includegraphics[width = 1\columnwidth]{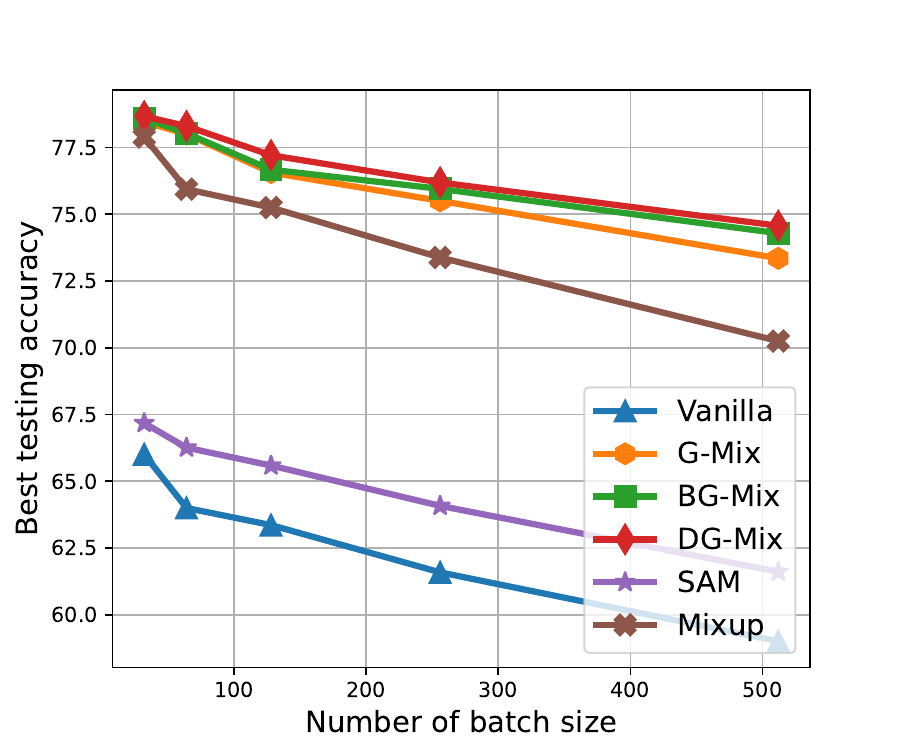}
		\caption{Cifar-100: ${\bf{Test}\textsubscript{Acc}}$}
		\label{fig:batch_acc_cifar100}
	\end{subfigure}
	\begin{subfigure}{0.24\columnwidth}
		\includegraphics[width = 1\columnwidth]{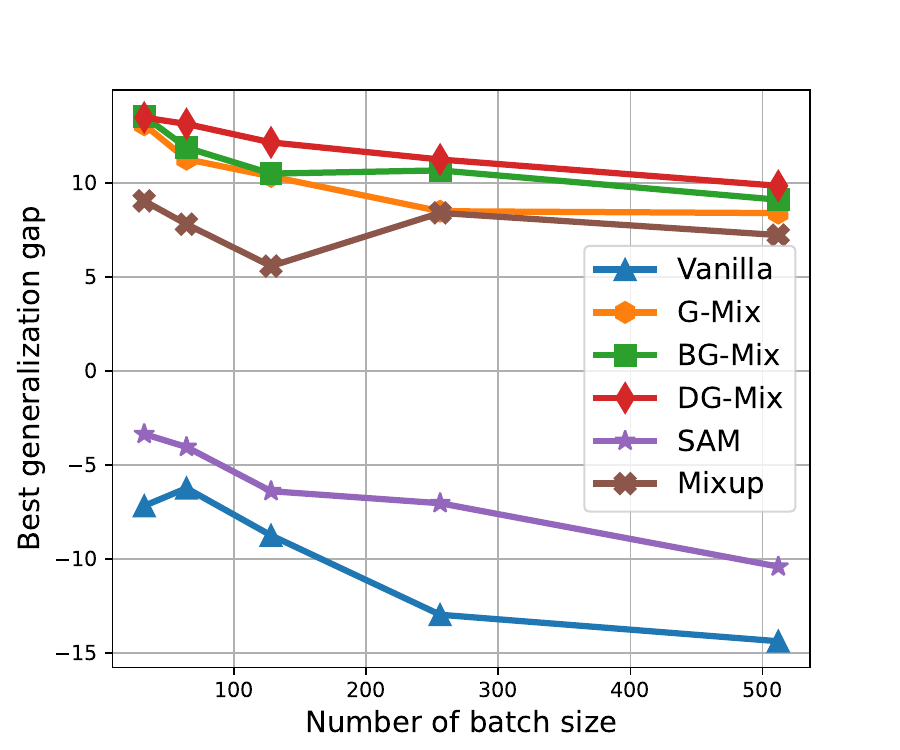}
		\caption{Cifar-100: ${\bf{Gap}}$}
		\label{fig:batch_gap_cifar100}
	\end{subfigure}
	%	\caption{text}
	\caption{Performance of the compared methods with the WRN architecture against STL-10 and Cifar-100 datasets.}
	\label{fig:batch}
\end{figure*}

\subsection{Datasets}
We have conducted experiments to evaluate the performance of the proposed learning methods on the following four datasets:
%\noindent\textbf
\begin{itemize}
	\item[1] \textbf{SVHN.} The Street View House Numbers (SVHN) \cite{netzer2011reading} is a real-world image dataset for classification and object recognition tasks. It contains $10$ classes of $73,257$ training and $26,032$ testing examples, where each example is formatted as $32 \times 32$ pixels. 
	\item[2] \textbf{Cifar-$\{10,100\}$.} Cifar-10 and Cifar-100 \cite{krizhevsky2009learning} datasets respectively have $10$ and $100$ classes of real-world images, where each image has $32 \times 32$ pixels. Specifically, the Cifar-10 dataset has $50,000$ in training and $10,000$ in testing images, with $6,000$ images per class. And each class in Cifar-100 dataset has $600$ images, which can be divided into $500$ training and $100$ testing examples. 
	\item[3] \textbf{STL-10.} The STL-10 \cite{coates2011analysis} is an image recognition dataset for developing supervised learning algorithms, which is inspired by the original Cifar-10 dataset but with some additional modifications. Particularly, the STL-10 dataset contains $10$ classes: airplane, bird, car, cat, deer, dog, horse, monkey, ship, and truck. Specifically, each class is with $500$ training and $800$ testing images, and each image has $96 \times 96$ color pixels, which makes STL-10 more challenging than Cifar-10.
	\item[4] \textbf{Tiny-ImageNet.} The ImageNet 2012 dataset \cite{krizhevsky2012imagenet} is a challenging image classification benchmark in the DL research field, which contains $1.2$ million training images and $50,000$ validation images from $1,000$ classes. In this paper, the introduced Tiny-ImageNet \cite{krizhevsky2012imagenet} is a subset of the original ImageNet 2012 \cite{krizhevsky2012imagenet} dataset, which has $200$ classes, where each class contains $500$ training, $50$ validation, and $50$ testing images that each image is of $64 \times 64$ pixels. 
\end{itemize}

\subsection{Experimental Setup}
Multiple DNN models have been trained using the proposed BG-Mix and DG-Mix algorithms under the developed G-Mix framework. Note that all the experiments are performed on the PyTorch \cite{paszke2017automatic} platform with one NVIDIA GTX 1080 Ti GPU card. 

\noindent\textbf{Models.}
For the SVHN dataset, we follow the settings in \cite{howard2017mobilenets} to train a MobileNet, which is considered to be an efficient model for mobile and embedded image classification applications that contain around $0.4$ million parameters. For Cifar-10, Cifar-100, STL-10, and Tiny-ImageNet datasets, we use the WideResNet (WRN) model \cite{zagoruyko2016wide}. Note that due to the computation capacity of our graphic card, we set the width factor to $4$ and the depth to $16$. Moreover, we also use ResNet-18 \cite{he2016deep} model against the Cifar-100 and Tiny-ImageNet benchmarks for performance comparison. 

\noindent\textbf{Training Details.}
The default implementation settings of the developed G-Mix learning framework are listed as follows: i) the number of training epochs $T$ is set to $100$ for SVHN, $200$ for Cifar10, Cifar-100, and STL-10, and $500$ for Tiny-ImageNet; ii) the training batch size is set to $|\mathcal{B}| = 128$; iii) the initialized learning rate $\eta$ is set to $0.1$ with two schedulers that perform $\eta = 0.1 \eta$ correspondingly on the $0.65T$-th and $0.85T$-th epoch; iv) the values of hyper-parameters introduced in the paper are defined as $\alpha=1$, $\rho=0.5$, and $\gamma = 0.5$, and their ablation studies are also performed to evaluate their impacts on the learning performance. 

\noindent\textbf{Evaluation Metrics.}
We compare the proposed BG-Mix and DG-Mix algorithms with the following baselines: Mixup, SAM, and G-Mix. Note that for better understanding, we also introduce the learning performance without any weight perturbation or data augmentation, called Vanilla. To provide a fair comparison and accurate evaluation, we measure the performance of the considered approaches with multiple metrics. Specifically, consider the model outputs respectively on training and testing as $\hat{\bf{y}}_{train}$ and $\hat{\bf{y}}_{test}$, and the true labels as ${\bf{y}}_{train}$ and ${\bf{y}}_{test}$, we introduce the following two numerical metrics besides the standard convergence and computation analysis: $\bf{Test\textsubscript{Acc}}$ and $\bf{Gap}$, where the higher value indicates to the better performance:  
\begin{itemize}
	\item $\bf{Test\textsubscript{Acc}} = 100\% \times \frac{|\hat{\bf{y}}_\text{test} == {\bf{y}}_\text{test}|}{|{\bf{y}}_\text{test}|}$. 
	\item $\bf{Gap} = 100\% \times ( \frac{|\hat{\bf{y}}_\text{test} == {\bf{y}}_\text{test}|}{|{\bf{y}}_\text{test}|} - \frac{|\hat{\bf{y}}_\text{train} == {\bf{y}}_\text{train}|}{|{\bf{y}}_\text{train}|})$.
\end{itemize}
Note that to present the results in the tables clearly, we mark the top-$3$ compared algorithms according to each corresponding metric: the $1$-st is marked as \textbf{bold} and $\underline{underline}$; the $2$-nd is marked as \textbf{bold}; and the $3$-rd is marked as $\underline{underline}$. 
\begin{table}[tb]
	\centering
	%	\begin{threeparttable}[tb]
		%		\begin{tablenotes}
			%			\item [1] Shorthand notation for the number of classes in one remote device. 
			%		\end{tablenotes}
		\begin{adjustbox}{width=0.85\columnwidth,center}
			\begin{tabular}{*{7}{l}}
				%		\toprule
				\toprule
				\multicolumn{1}{l}{\bf{}} & \multicolumn{2}{c}{\bf{SVHN}} & \multicolumn{2}{c}{\bf{STL-10}} & \multicolumn{2}{c}{\bf{Cifar-10}} \\
				\bf{Method}& $\bf{Test\textsubscript{Acc}}$    & \bf{Gap} & $\bf{Test\textsubscript{Acc}}$   & \bf{Gap} & $\bf{Test\textsubscript{Acc}}$     & \bf{Gap} \\
				\midrule
				{Vanilla }&  93.24 & 19.78 &  64.76 & -9.91 &  88.61 & 12.13\\
				{SAM} & 94.44  & 21.34 &  66.41 & -6.36 &   89.89  & 14.09 \\
				\midrule 
				{Mixup }& 95.14 & 24.38 &  79.35 & 6.08 &  94.33 & 19.19 \\
				{G-Mix} &  $\underline{95.67}$  & $\underline{\bf{27.92}}$ &  $\underline{81.40}$ & $\underline{6.96}$ &  $\underline{94.65}$  & $\underline{20.58}$  \\
				%		\midrule
				{BG-Mix } & $\bf{95.83}$ & $\underline{27.09}$ &  $\bf{82.15}$ & $\bf{8.46}$ &  $\bf{95.10}$ & $\bf{20.94}$  \\
				{DG-Mix } & $\underline{\bf{95.87}}$   & $\bf{27.83}$ & $\underline{\bf{82.45}}$ & $\underline{\bf{9.39}}$ &  $\underline{\bf{95.33}}$   & $\underline{\bf{21.25}}$ \\
				\bottomrule
			\end{tabular} 
		\end{adjustbox}
		\caption{Performance comparison multiple datasets.}
		%	\end{threeparttable}
	\label{Tab:comparison_small}
\end{table}
\subsection{Result Analysis}
\noindent\textbf{Convergence.}\label{Subsec:convergence}
We first analyze the convergence of learning on both the testing accuracy and the generalization gap, which show their experimental results in Fig.~\ref{fig:convergence} and \ref{fig:gap}, respectively. The testing accuracy curves in Fig.~\ref{fig:convergence} show that all of the considered training approaches provide a convergence guarantee under the experimental settings in this paper. It can be easily noticed that compared to the existing Mixup, SAM, and Vanilla methods, both the developed G-Mix framework and the proposed BG- and DG-Mix algorithms achieve a better testing accuracy convergence. Specifically, the performance of the introduced methods follows a stable relationship: DG-Mix $>$ BG-Mix $>$ G-Mix $\gg$ Mixup $\gg$ SAM $>$ Vanilla. 

For example, to achieve $70\%$ testing accuracy on the STL-10 benchmark, the proposed DG-Mix algorithm takes only $18$ training epochs, where $61$, $65$, and $87$ rounds for BG-Mix, G-Mix, and Mixup approaches, respectively. Additionally, we can also notice that compared to the difference in testing accuracy, the results of the generalization gap in Fig.~\ref{fig:gap} can be more significant. For example, though the proposed DG-Mix and BG-Mix algorithms achieve a similar testing accuracy on the Cifar-100 dataset against the G-Mix framework, the corresponding generalization gap results in Fig.~\ref{fig:gap_cifar100} indicate their superiority. Another interesting phenomenon is that SAM achieves a clear better generalization gap against the Vanilla, which supports the claim in \cite{foret2020sharpness} that SAM indeed leverages a better generalization ability. 

\noindent\textbf{Comparison Against Introduced Methods.}
We then summarize the performance comparison results for the compared methods on SVHN, STL-10, Cifar-10, Cifar-100, and TinyImageNet datasets in Table.~\ref{Tab:comparison_small} and \ref{Tab:comparison_big}. Note that the ``${\bf{Test}\textsubscript{Acc}}$" column indicates the best achieved testing accuracy and the $``\bf{Gap}"$ represents the best achieved generalization gap, respectively during the training process under the developed G-Mix framework. Note that for the tables in this paper, to properly illustrate the results, we mark the top-3 performance of introduced methods with the same training model according to each evaluation metric: 1-st as $\bf{bold}$ and $\underline{underline}$, 2-nd as $\bf{bold}$, and 3-rd as $\underline{underline}$. 
We can notice from the results in Table.~\ref{Tab:comparison_small} that, generally, both the proposed BG-Mix and DG-Mix algorithms and the developed G-Mix framework outperform the existing SAM and Mixup methods. Specifically, the DG-Mix reaches the best overall performance on both the best achieved testing accuracy and the generalization gap, and also the relationship mentioned in the convergence analysis is also supported. It can be observed that compared to the Mixup technique, the proposed DG-Mix algorithm achieves $0.77\%$ higher accuracy ($95.87\%$ versus $95.14\%$) and $14.2\%$ higher generalization gap ($27.83\%$ versus $24.38\%$). The advantage can be more obvious on the Cifar-100 dataset, where the DG-Mix algorithm outperforms Mixup by $3.9\%$ and $54.4\%$ higher accuracy and gap, respectively.   

Additionally, the results in Table.~\ref{Tab:comparison_big} show that the improvement of DG- and BG-Mix algorithms under the developed G-Mix framework can be consistent at different architectures. For example, for the Cifar-100 dataset, the BG-Mix (${\bf{Test}\textsubscript{Acc}}$ of $78.05\%$ and $``\bf{Gap}"$ of $8.21\%$) outperforms the Mixup method (${\bf{Test}\textsubscript{Acc}}$ of $76.88\%$ and $``\bf{Gap}"$ of $6.31\%$) with the ResNet-18 architecture, and the DG-Mix (${\bf{Test}\textsubscript{Acc}}$ of $77.21\%$ and $``\bf{Gap}"$ of $12.15\%$) outperforms the Mixup method (${\bf{Test}\textsubscript{Acc}}$ of $75.25\%$ and $``\bf{Gap}"$ of $5.57\%$) with the WRN architecture. Moreover, for the Tiny-ImageNet dataset, the DG-Mix (${\bf{Test}\textsubscript{Acc}}$ of $60.91\%$ and $``\bf{Gap}"$ of $-4.01\%$) outperforms the Mixup method (${\bf{Test}\textsubscript{Acc}}$ of $58.31\%$ and $``\bf{Gap}"$ of $-9.04\%$) with the WRN architecture, and correspondingly  (${\bf{Test}\textsubscript{Acc}}$ of $60.00\%$ and $``\bf{Gap}"$ of $-8.83\%$) against (${\bf{Test}\textsubscript{Acc}}$ of $59.01\%$ and $``\bf{Gap}"$ of $-12.01\%$) with the ResNet-18 architecture. 

Meanwhile, the results show that though we reduce the width and depth of the original WideResNet architecture in \cite{zagoruyko2016wide}, it still achieves a competitive ${\bf{Test}\textsubscript{Acc}}$ against the ResNet-18 architecture. Another interesting phenomenon is that compared to the ResNet-18 architecture, the overall generalization gap is significantly higher with the WRN architecture as shown in Table.~\ref{Tab:comparison_big}. We consider this might be because of the fact that a wider DNN model architecture can probably provide a better generalization ability. 
\begin{table}[tb]
	\centering
	%	\begin{threeparttable}[tb]
		%		\begin{tablenotes}
			%			\item [1] Shorthand notation for the number of classes in one remote device. 
			%		\end{tablenotes}
		\begin{adjustbox}{width=0.85\columnwidth,center}
			\begin{tabular}{*{6}{l}}
				%		\toprule
				\toprule
				\multicolumn{1}{l}{\bf{}} & \multicolumn{1}{l}{}  & \multicolumn{2}{c}{\bf{Cifar-100}} & \multicolumn{2}{c}{\bf{Tiny-ImageNet}} \\
				\bf{Method} &\bf{Model} & $\bf{Test\textsubscript{Acc}}$   & \bf{Gap} & $\bf{Test\textsubscript{Acc}}$    & \bf{Gap} \\
				\midrule
				\multirow{2}{*}{Vanilla } & WRN  &  63.36 & -8.75 & 42.70 & -28.91\\
				& ResNet18 & 57.11 & -17.60 & 37.91 & -35.12\\
				\multirow{2}{*}{SAM} & WRN & 65.59 & -6.40 & 45.69 & -25.12\\
				& ResNet18 & 60.46 & -10.67 & 41.06 & -31.74\\
				\midrule 
				\multirow{2}{*}{Mixup } & WRN& 75.25 & 5.57 & 58.31 & -9.04  \\
				& ResNet18 &  76.68 & 6.31& 59.01 & -12.01  \\
				\multirow{2}{*}{G-Mix} & WRN &  $\underline{76.56}$  & $\underline{10.34}$ &  $\underline{60.05}$  & $\bf{-5.48}$   \\
				& ResNet18&  $\underline{77.79}$ & $\underline{7.09}$  &  $\bf{59.95}$ & $\bf{-9.55}$ \\
				%		\midrule
				\multirow{2}{*}{BG-Mix } & WRN & ${\bf{76.67}}$ & $\bf{10.49}$ & $\bf{60.52}$ & $\underline{-5.68}$  \\
				& ResNet18 & $\underline{\bf{78.05}}$ & $\bf{8.21}$  & $\underline{59.70}$  & $\underline{-9.59}$\\
				\multirow{2}{*}{DG-Mix } & WRN & $\underline{\bf{77.21}}$   & $\underline{\bf{12.15}}$ & $\underline{\bf{60.91}}$   & $\underline{\bf{-4.01}}$ \\
				& ResNet18 & $\bf{77.90}$   & $\underline{\bf{8.54}}$  & $\underline{\bf{60.00}}$   & $\underline{\bf{-8.83}}$ \\
				\bottomrule
			\end{tabular} 
		\end{adjustbox}
		\caption{Performance comparison of more challenging datasets on different DNN models.}
		\label{Tab:comparison_big}
		%	\end{threeparttable}
\end{table}

\begin{figure*}[tb]
	\centering
	\begin{subfigure}{0.24\columnwidth}
		\includegraphics[width = 1\columnwidth]{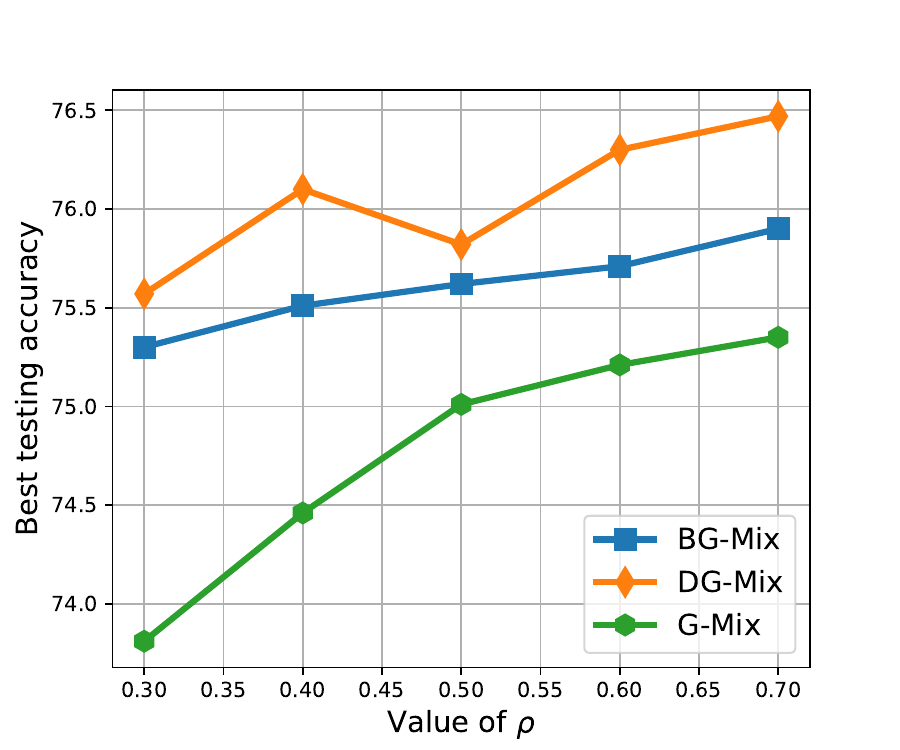}
		\caption{STL-10: ${\bf{Test}\textsubscript{Acc}}$}
		\label{fig:hyper_rho_acc_stl10}
	\end{subfigure}
	\begin{subfigure}{0.24\columnwidth}
		\includegraphics[width = 1\columnwidth]{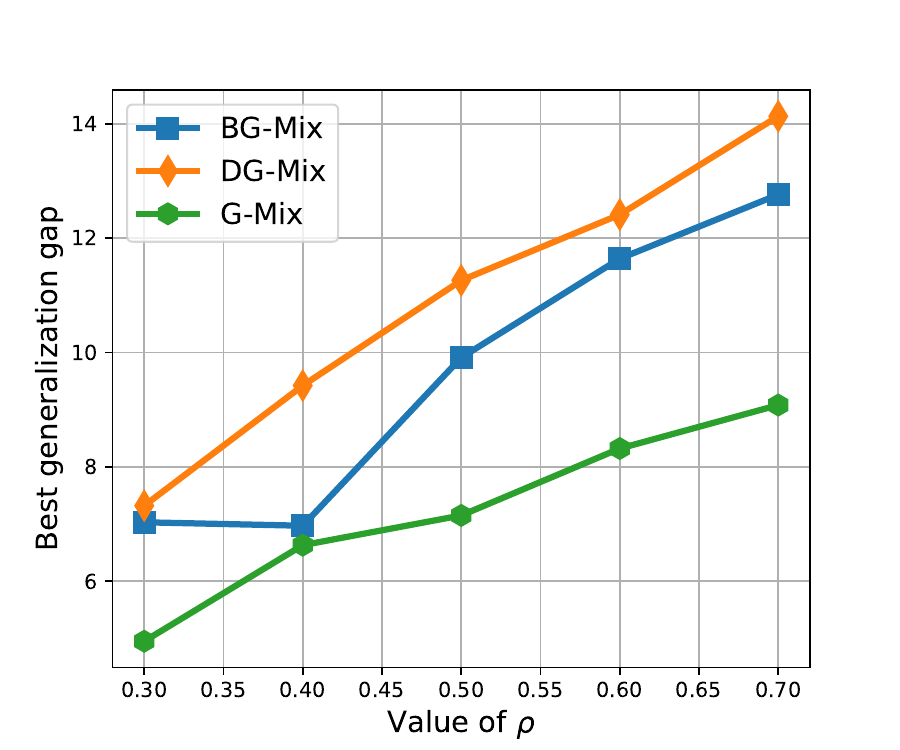}
		\caption{STL-10: ${\bf{Gap}}$}
		\label{fig:hyper_rho_gap_stl10}
	\end{subfigure}
	\begin{subfigure}{0.24\columnwidth}
		\includegraphics[width = 1\columnwidth]{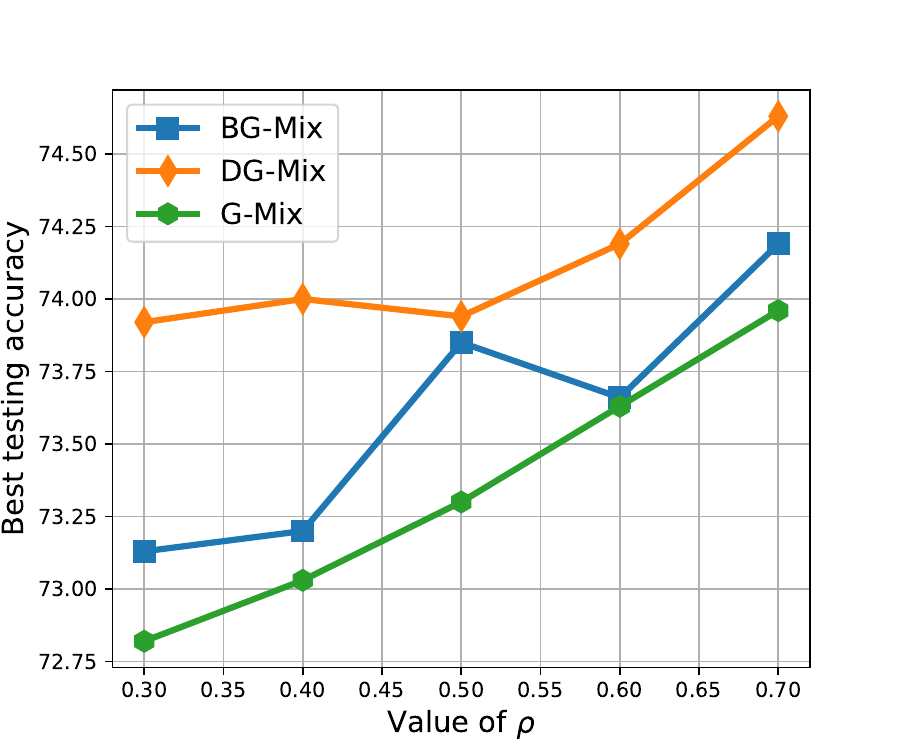}
		\caption{Cifar-100: ${\bf{Test}\textsubscript{Acc}}$}
		\label{fig:hyper_rho_acc_cifar100}
	\end{subfigure}
	\begin{subfigure}{0.24\columnwidth}
		\includegraphics[width = 1\columnwidth]{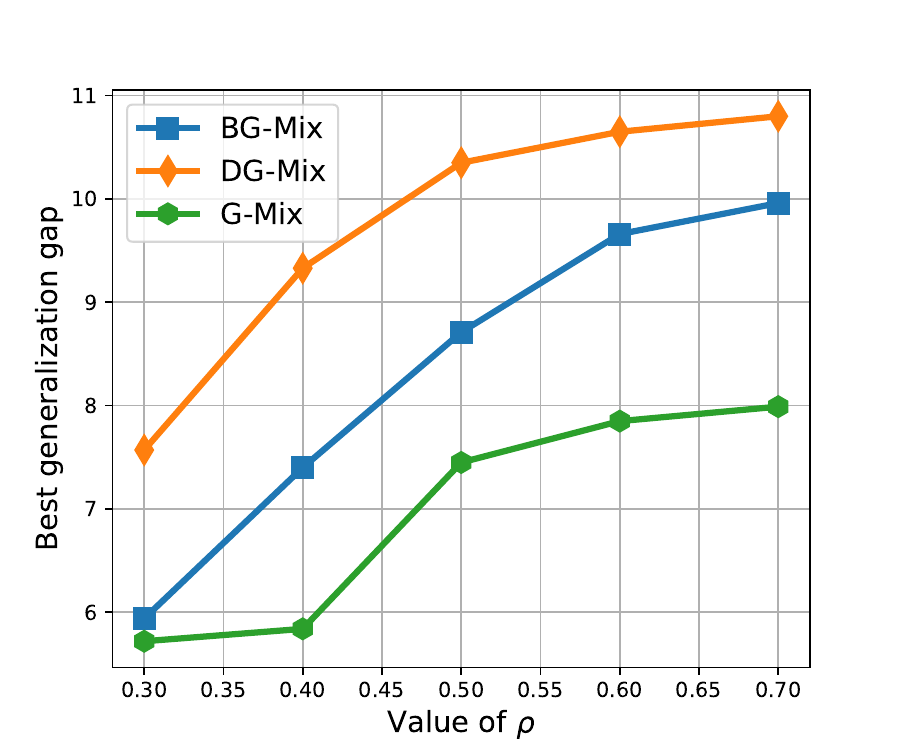}
		\caption{Cifar-100: ${\bf{Gap}}$}
		\label{fig:hyper_rho__gap_cifar100}
	\end{subfigure}
	%	\caption{text}
	\caption{Impact of $\rho$ to the proposed BG-Mix and DG-Mix algorithms against the STL-10 and Cifar-100 datasets under the developed G-Mix framework, $\gamma = 0.5$, $|\mathcal{B}|= 512$.}
	\label{fig:hyper_rho}
\end{figure*}

\begin{figure*}[tb]
	\centering
	\begin{subfigure}{0.24\columnwidth}
		\includegraphics[width = 1\columnwidth]{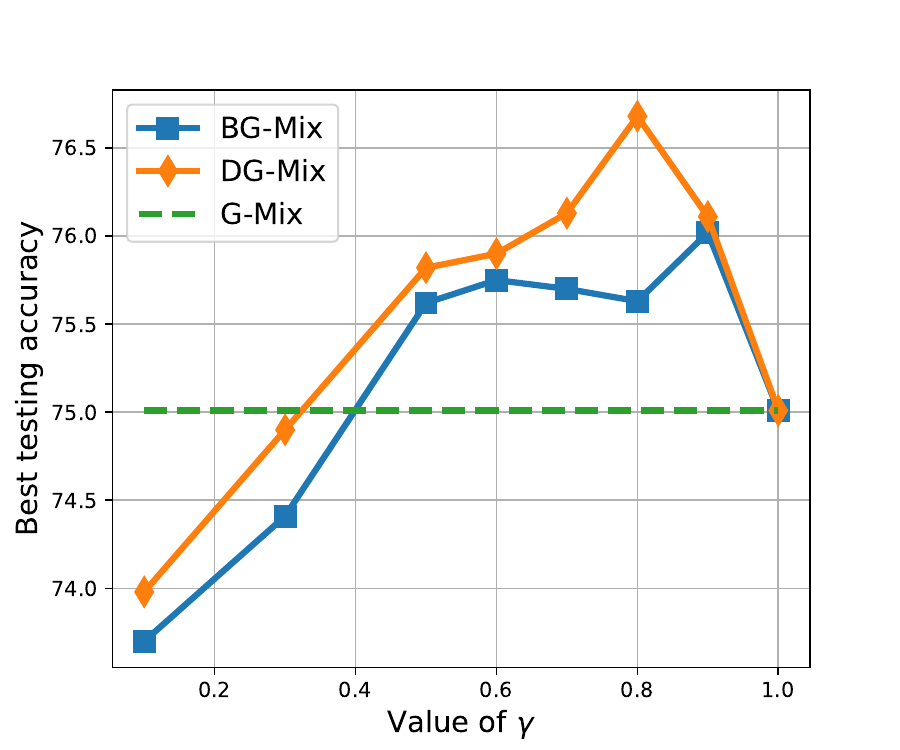}
		\caption{STL-10: ${\bf{Test}\textsubscript{Acc}}$}
		\label{fig:hyper_gamma_acc_stl10}
	\end{subfigure}
	\begin{subfigure}{0.24\columnwidth}
		\includegraphics[width = 1\columnwidth]{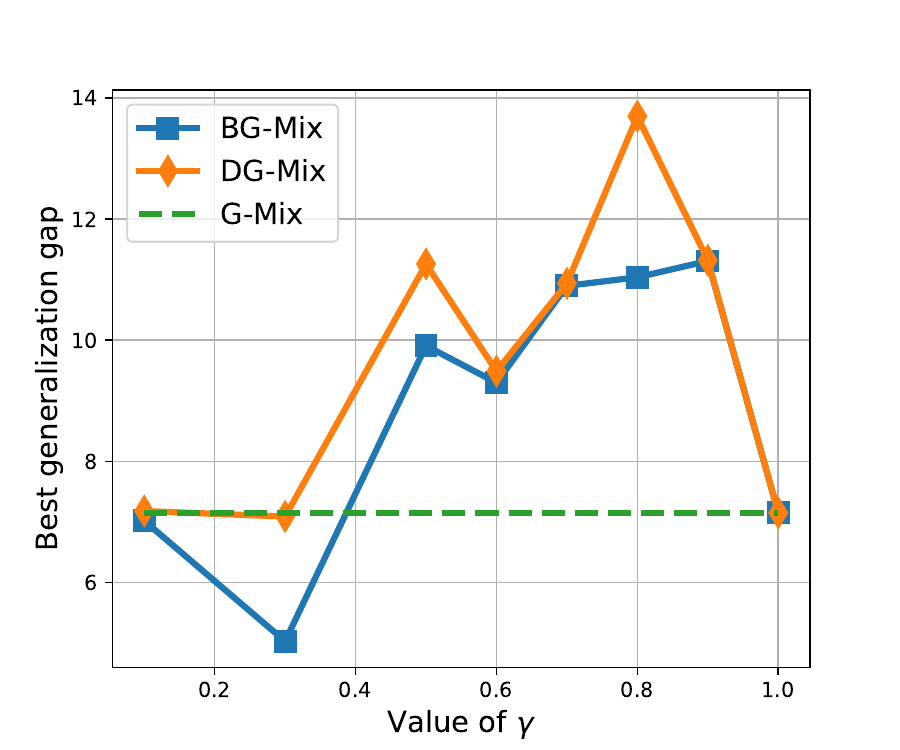}
		\caption{STL-10: ${\bf{Gap}}$}
		\label{fig:hyper_gamma_gap_stl10}
	\end{subfigure}
	\begin{subfigure}{0.24\columnwidth}
		\includegraphics[width = 1\columnwidth]{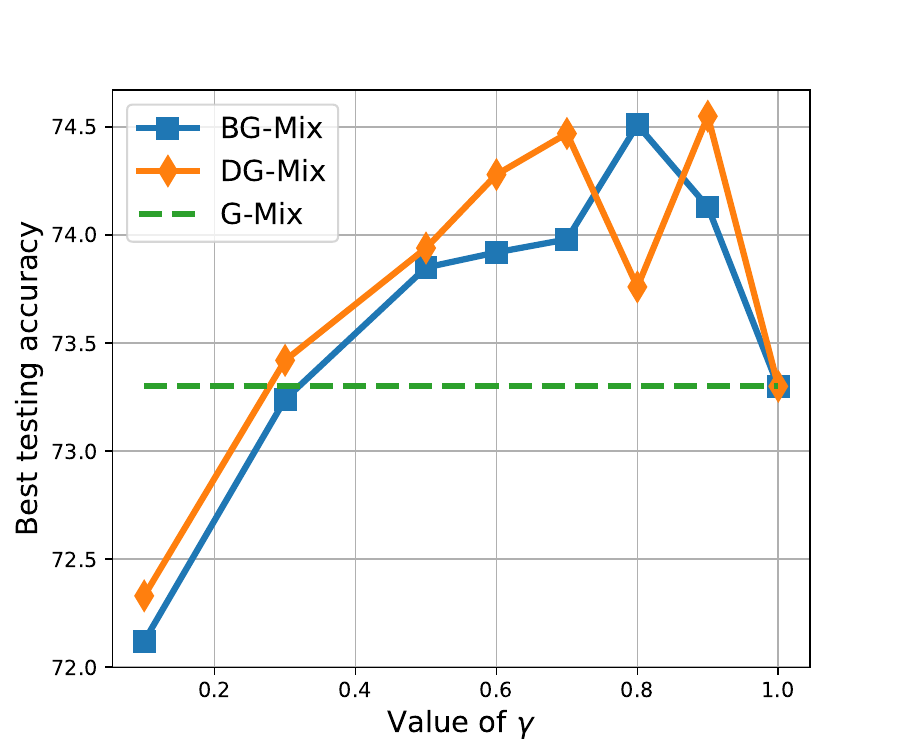}
		\caption{Cifar-100: ${\bf{Test}\textsubscript{Acc}}$}
		\label{fig:hyper_gamma_acc_cifar100}
	\end{subfigure}
	\begin{subfigure}{0.24\columnwidth}
		\includegraphics[width = 1\columnwidth]{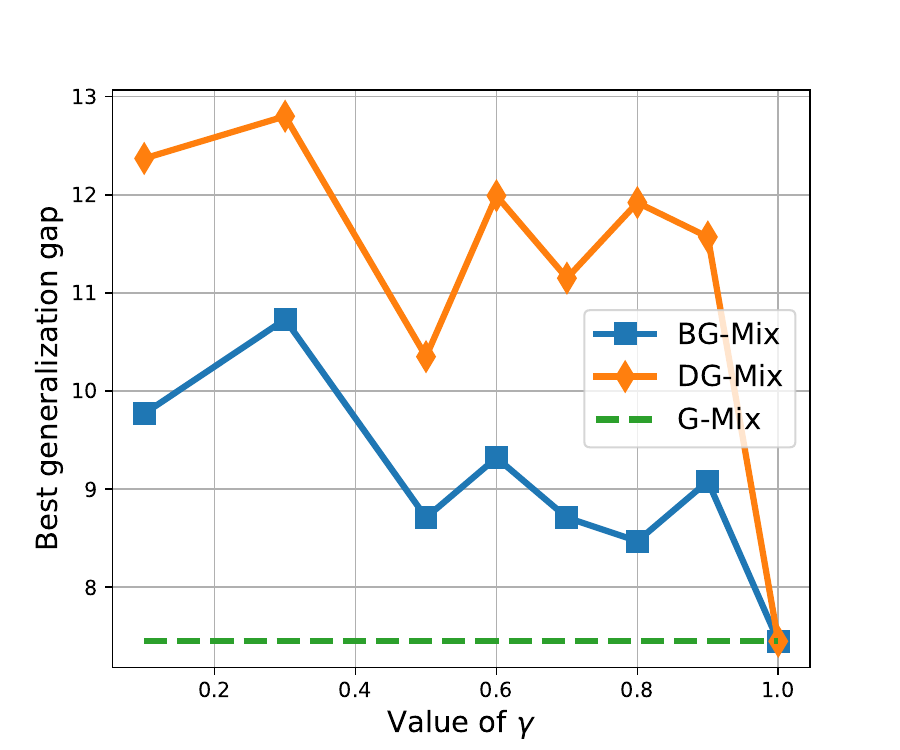}
		\caption{Cifar-100: ${\bf{Gap}}$}
		\label{fig:hyper_gamma__gap_cifar100}
	\end{subfigure}
	%	\caption{text}
	\caption{Impact of $\gamma$ to the proposed BG-Mix and DG-Mix algorithms against the STL-10 and Cifar-100 datasets under the developed G-Mix framework, $\rho = 0.5$, $|\mathcal{B}|= 512$.}
	\label{fig:hyper_gamma}
\end{figure*}

\noindent\textbf{Impact of Batch Size.}
Next, we vary the size of the training batch from $32$ to $512$ for the compared methods with the WRN architecture on the STL-10 and Cifar-100 datasets. The results illustrated in Fig.~\ref{fig:batch} support our intuition that the performance of both the best testing accuracy and the generalization gap of the compared learning approaches satisfies an inverse relationship toward the size of the training batch. Under this condition, it can be noticed that the proposed algorithms in this paper are more robust against the existing Mixup and SAM methods. For example, on the Cifar-100 dataset, the reduction of ${\bf{Test}\textsubscript{Acc}}$ is $5.2\%$ (from $78.67\%$ to $74.57\%$), which outperforms the Mixup method with $9.8\%$ (from $77.89\%$ to $70.25\%$). 

Moreover, we also notice that there is a significant performance gap between those methods without and with the Mixup technique. For example, the generalization gaps of Vanilla and SAM against the STL-10 dataset are negative, while the other methods all achieve positive results. An intuitive explanation for this empirical phenomenon is that compared to the SAM method, the Mixup technique might play a more important role in the DNN model generalization ability improvement among the considered learning benchmarks in this paper. 

\noindent\textbf{Impact of $\rho$.}
In Fig.~\ref{fig:hyper_rho}, we demonstrate the performance of the proposed BG-Mix and DG-Mix algorithms with a varying $\rho$ value from $0.3$ to $0.7$, compared to the developed G-Mix framework. The results shown in Fig.~\ref{fig:hyper_rho} indicate a stable relationship between the compared methods in this experimental setting: DG-Mix $>$ BG-Mix $>$ G-Mix. Meanwhile, we can also notice that there is a slight positive relationship between the value of $\rho$ and both the best achieved ${\bf{Test}\textsubscript{Acc}}$ and ${\bf{Gap}}$. For example, on the STL-10 dataset, the best achieved testing accuracy of the DG-Mix algorithm increases $1.20\%$ (from $75.57\%$ to $76.47$). And on the Cifar-100 dataset, the ${\bf{Test}\textsubscript{Acc}}$ of the proposed BG-Mix algorithm increases $1.45\%$ (from $73.13\%$ to $74.19\%$). Moreover, such improvement is more obvious on the performance of ${\bf{Gap}}$, e.g., the best achieved generalization gap of DG-Mix algorithm against Cifar-100 dataset increases $42.67\%$ (from $7.57\%$ to $10.80\%$). 

\noindent\textbf{Impact of $\gamma$.}
In Fig.~\ref{fig:hyper_gamma}, we investigate the performance of the proposed BG-Mix and DG-Mix algorithms with a varying value of $\gamma$ from $0.1$ to $0.9$, under the developed G-Mix framework where $\rho = 0.5$ and $|\mathcal{B}| = 512$. Note that the green dotted line in Fig.~\ref{fig:hyper_gamma} indicates the performance of the G-Mix framework, whose learning result is not influenced by the change of $\gamma$ and could be viewed as a special case of the proposed BG-Mix and DG-Mix algorithms with $\gamma=1$. The results show that both the proposed BG-Mix and DG-Mix algorithms outperform the G-Mix framework in most cases. Specifically, it can be noticed that $\gamma = 0.3$ is the minimum value that is optimal for achieving a comparable performance against G-Mix. Furthermore, the results of ${\bf{Test}\textsubscript{Acc}}$ against STL-10 and Cifar-100 datasets indicate that when $\gamma \approx 0.8$, the DG-Mix might achieve the best ${\bf{Test}\textsubscript{Acc}}$, and when $\gamma \approx 0.9$, the BG-Mix achieves the best ${\bf{Test}\textsubscript{Acc}}$. Additionally, though the ${\bf{Gap}}$ results on Cifar-100 indicate a slightly reverse relationship between the value of $\gamma$ and the beat achieved generalization gap, both the BG- and DG-Mix also outperform G-Mix on each considered scenario.

\begin{table}[tb]
	\centering
	\begin{adjustbox}{width=0.85\columnwidth,center}
		\begin{tabular}{*{6}{l}}
			%		\toprule
			\toprule
			%			\multicolumn{1}{l}{\bf{}} & \multicolumn{1}{l}{}  & \multicolumn{2}{c}{\bf{STL-10}} & \multicolumn{2}{c}{\bf{Cifar-100}} \\
			\bf{} & \bf{SVHN} & \bf{STL-10}   & \bf{Cifar-10} & \bf{Cifar-100}    & \bf{Tiny-ImageNet} \\
			\midrule
			{Vanilla } & ${787.6}$ &  ${1523.7}$  & 4404.8  & 4781.9& 23528.2\\
			{Mixup }& ${826.3}$ & ${1529.9}$& 4444.1 & 4791.8  & 23553.0  \\
			\midrule 
			{SAM}& $\bf{\underline{1157.4}}$ & $\bf{2407.2}$  & $\underline{8468.3}$ & $\bf{9109.2}$ & $\bf{45544.3}$\\
			{G-Mix} &$\underline{1198.5}$ &  $\underline{2409.1}$ &  $\bf{8466.2}$  & $\underline{9114.7}$  &  $\underline{45742.9}$   \\
			%		\midrule
			{BG-Mix } &$\bf{1186.3}$ &  $\bf{\underline{2376.0}}$ & $\bf{\underline{8465.4}}$ &  $\bf{\underline{9079.5}}$  & $\bf{\underline{45294.8}}$ \\
			{DG-Mix }&1320.4 & 2577.5 & 9164.8   & 9389.1  & $49452.4$ \\
			\bottomrule
		\end{tabular} 
	\end{adjustbox}
	\caption{Performance comparison of the computation cost from running time (seconds) against multiple datasets.}
	\label{Tab:comparison_time}
	%	\end{threeparttable}
\end{table}

\noindent\textbf{Study of Computation Cost.}
To this end, Table.~\ref{Tab:comparison_time} shows the experimental results of the running time (seconds) for each compared method during the training process to converge with the default settings. Note that the marks in Table.~\ref{Tab:comparison_time} only consider the SAM-related methods, and the extra computation cost of SAM against standard SGD illustrated in \cite{foret2020sharpness} is not our concentration. It can be noticed from the results that compared to Vanilla and Mixup, the computation cost of the SAM-related approaches is almost doubled, we consider this might be because of the extra back-propagation operation of SAM.

Additionally, we can notice that the proposed BG-Mix algorithm achieves minimal computation cost, compared to DG-Mix and G-Mix, because ignoring the subset of training samples in $\mathcal{B}$ can improve the learning efficiency. Moreover, though the DG-Mix algorithms achieve the worst computation cost, the extra time is only linear and limited considering the doubled cost because of the applied SAM technique.

\section{Conclusion}\label{Sec:conclusion}
In this paper, we introduce the Generalized-Mixup (G-Mix) learning framework, which enhances the generalization ability of deep neural networks (DNNs). The G-Mix framework consists of a two-step learning strategy that combines the strengths of both the Sharpness-Aware Minimization (SAM) and Mixup techniques. In the first step, the training data is linearly interpolated to expand the training loss principle from Empirical Risk Minimization (ERM) to Vicinal Risk Minimization (VRM). Subsequently, the model is updated by incorporating a worst-case perturbation through an additional back-propagation operation. Through theoretical analysis, we provide insights into how the G-Mix framework improves the generalization ability of DNN models. To further optimize the performance within the G-Mix framework, we propose two novel algorithms, Binary G-Mix (BG-Mix) and Decomposed G-Mix (DG-Mix). Leveraging the two-step structure of G-Mix, we generate a Sharpness-sensitivity vector that evaluates the loss changes between the two back-propagation operations. This vector guides the weightiness of each example in the training batch during the Mixup process. Specifically, the BG-Mix algorithm excludes less sensitive examples based on an empirical threshold, which improves learning efficiency. On the other hand, the DG-Mix algorithm decomposes the loss of less sensitive examples, utilizing the orthogonal component to enhance performance. Extensive experimental analysis is conducted to validate the effectiveness of both the proposed algorithms and the developed G-Mix framework across multiple real-world learning applications. The results demonstrate the superior performance of our approaches in terms of generalization ability and learning efficiency.

\bibliographystyle{IEEEtran}  
\bibliography{main_arxiv}

\end{document}